\documentclass[twoside,11pt]{article}

\usepackage{jmlr2e}
\usepackage{graphicx} 
\usepackage{subfigure} 
\usepackage{natbib}
\usepackage{bm}

\usepackage{algorithm}
\usepackage{color}
\usepackage{pbox}
\usepackage{algorithmic}

\usepackage{url} 
\usepackage{bm}
\usepackage{bbm}
\usepackage{amsmath,amssymb}
\usepackage{xspace}
\usepackage{dsfont}
\usepackage{psfrag}
\usepackage{rotating}
\usepackage{epstopdf}
\usepackage{booktabs}
\usepackage{array}
\usepackage{cleveref}

\newcommand{\OmegaSpace}{\Omega}

\newcommand{\transp}{^\top}

\newcommand{\proba}{P}
\newcommand{\grad}{\nabla}

\newcommand{\U}{\bm{U}}

\newcommand{\complexSpace}{\mathbb{C}}
\renewcommand{\Re}{\mathbb{R}}
\newcommand{\C}{\complexSpace} 
\newcommand{\R}{\Re} 
\newcommand{\real}{\mathrm{Re}}
\newcommand{\imag}{\mathrm{Im}}

\newcommand{\srank}{\mathrm{rank}_{\pm}}

\newcommand{\sign}{\mathrm{sign}}

\providecommand{\U}[1]{\protect\rule{.1in}{.1in}}
\newcommand{\be}{\begin{equation}}
\newcommand{\ee}{\end{equation}}
\newcommand{\bd}{\begin{definition}}
\newcommand{\ed}{\end{definition}}
\newcommand{\ba}{\begin{algorithm}}
\newcommand{\ea}{\end{algorithm}}
\newcommand{\br}{\begin{problem}}
\newcommand{\er}{\end{problem}}
\newcommand{\bex}{\begin{example}}
\newcommand{\eex}{\end{example}}
\newcommand{\bt}{\begin{theorem}}
\newcommand{\et}{\end{theorem}}

\def\ifa{\iffalse}
\def\ifappendix{\iftrue}

\newcommand{\Relation}{X}
\newcommand{\ObsTensor}{\mathbf{Y}}
\newcommand{\ScoreTensor}{\mathbf{X}}
\newcommand{\EntitySpace}{\mathcal{E}}
\newcommand{\RelationSpace}{\mathcal{R}}
\newcommand{\TripleSpace}{\mathcal{T}}

\renewcommand{\cite}{\citep}

\newcommand{\Mnr}{\R^{n \times n}}
\newcommand{\Mnc}{\C^{n \times n}}

\newcommand{\rrank}{\mathrm{rank}}
\newcommand{\T}{\transp}
\newcommand{\pT}{^{\prime\top}}

\newtheorem{thm}{Theorem}
\newtheorem{defn}{Definition}
\newtheorem{cor}{Corollary}

\newcommand{\setent}{\mathcal{E}} 
\newcommand{\setrel}{\mathcal{R}} 
\newcommand{\Z}{Z} 
\newcommand{\X}{X} 
\newcommand{\eemb}{e} 
\newcommand{\Eemb}{E} 
\newcommand{\wemb}{w} 
\newcommand{\Wemb}{W} 
\newcommand{\rank}{K}

\newcommand{\icol}[1]{
  \left(\begin{smallmatrix}#1\end{smallmatrix}\right)%
}

\newsavebox{\vecs}
\savebox{\vecs}{$\left(\begin{smallmatrix}1\\-2\end{smallmatrix}\right)$}
\newsavebox{\veco}
\savebox{\veco}{$\left(\begin{smallmatrix}-3\\1\end{smallmatrix}\right)$}

\usepackage[textsize=small]{todonotes}

\def\tt{\texttt}
\graphicspath{{figures/}}

\jmlrheading{18}{2017}{1-38}{11/16; Revised 7/17}{11/17}{16-563}{Th\'eo Trouillon, Christopher R. Dance, Johannes Welbl, Sebastian Riedel, \'Eric Gaussier and Guillaume Bouchard}

\ShortHeadings{Knowledge Graph Completion via Complex Tensor Factorization}{Trouillon, Dance, Gaussier, Welbl, Riedel and Bouchard}
\firstpageno{1}

\begin{document}

\title{Knowledge Graph Completion via Complex Tensor Factorization}

\author{\name Th\'eo Trouillon \email theo.trouillon@imag.fr \\
       \addr Univ. Grenoble Alpes, 700 avenue Centrale, 38401 Saint Martin d'H\`eres, France
       \AND
       \name Christopher R. Dance \email chris.dance@xrce.xerox.com \\
       \addr NAVER LABS Europe, 6 chemin de Maupertuis, 38240 Meylan, France
       \AND
       \name \'Eric Gaussier \email eric.gaussier@imag.fr \\
       \addr Univ. Grenoble Alpes, 700 avenue Centrale, 38401 Saint Martin d'H\`eres, France
       \AND
       \name Johannes Welbl \email j.welbl@cs.ucl.ac.uk \\
       \name Sebastian Riedel \email s.riedel@cs.ucl.ac.uk \\
       \addr University College London, Gower St, London WC1E 6BT, United Kingdom 
       \AND
       \name Guillaume Bouchard \email g.bouchard@cs.ucl.ac.uk \\
       \addr Bloomsbury AI, 115 Hampstead Road, London NW1 3EE, United Kingdom \\
       \addr University College London, Gower St, London WC1E 6BT, United Kingdom
       }

\editor{Luc De Raedt}

\maketitle

\begin{abstract}

In statistical relational learning, knowledge graph completion deals with automatically understanding the structure of large knowledge graphs---labeled directed graphs---and predicting missing relationships---labeled edges. State-of-the-art embedding models propose different trade-offs between modeling expressiveness, and time and space complexity. We reconcile both expressiveness and complexity through the use of complex-valued embeddings and explore the link between such complex-valued embeddings and unitary diagonalization. We corroborate our approach theoretically and show that \emph{all} real square matrices---thus all possible relation/adjacency matrices---are the real part of some unitarily diagonalizable matrix. This results opens the door to a lot of other applications of square matrices factorization. Our approach based on complex embeddings is arguably simple, as it only involves a Hermitian dot product, the complex counterpart of the standard dot product between real vectors, whereas other methods resort to more and more complicated composition functions to increase their expressiveness. The proposed complex embeddings are scalable to large data sets as it remains linear in both space and time, while consistently outperforming alternative approaches on standard link prediction benchmarks.\footnote{Code is available at: \url{https://github.com/ttrouill/complex}}
\end{abstract}

\begin{keywords}
complex embeddings, tensor factorization, knowledge graph, matrix completion, statistical relational learning
\end{keywords}

\section{Introduction}

Web-scale knowledge graph provide a structured representation of world knowledge, with projects such as the Google Knowledge Vault~\cite{Dong:2014:KnowledgeVault}. 
They enable a wide range of applications including recommender systems, question answering and automated personal agents. The incompleteness of these knowledge graphs---also called knowledge bases---has stimulated research into predicting missing entries, a task known as \emph{link prediction} or \emph{knowledge graph completion}. 
The need for high quality predictions required by link prediction applications made 
it progressively become the main problem in statistical relational learning \cite{Getoor2007}, 
a research field interested in relational data representation and modeling.




Knowledge graphs were born with the advent of the Semantic Web, pushed
by the World Wide Web Consortium (W3C) recommendations. Namely,
the Resource Description Framework (RDF) standard, that underlies
knowledge graphs' data representation, provides for the first time
a common framework across all connected information systems
to share their data under the same paradigm. 
Being more expressive than classical relational databases,
all existing relational data can be translated into RDF knowledge graphs \cite{sahoo2009survey}.

Knowledge graphs express data as a directed graph with labeled edges (relations) between nodes (entities). Natural redundancies between the recorded relations often make it possible to fill in the missing entries of a knowledge graph. 
As an example, the relation \tt{CountryOfBirth} could not be recorded for all entities, 
but it can be inferred if the relation \tt{CityOfBirth} is known. 
The goal of link prediction is the automatic discovery of such regularities. However, many relations are non-deterministic:  the combination of the two facts \tt{IsBornIn(John,Athens)} and \tt{IsLocatedIn(Athens,Greece)} does not always imply the fact \tt{HasNationality(John,Greece)}. 
Hence, it is natural to handle inference probabilistically, and jointly with 
other facts involving these relations and entities.
To this end, an increasingly popular method is to state the knowledge 
graph completion task as a 3D binary tensor completion problem, where 
each tensor slice is the adjacency matrix of one relation in the 
knowledge graph, and compute a decomposition of this partially-observed tensor
from which its missing entries can be completed.

Factorization models with low-rank embeddings were popularized by the Netflix challenge \cite{koren_netflix}. A partially-observed matrix or tensor is decomposed into a product of embedding matrices with much smaller dimensions, resulting in fixed-dimensional vector representations 
for each entity and relation in the graph, that allow completion of the missing entries.
For a given fact \emph{r(s,o)} in which the subject entity $s$ is linked to the object entity $o$ through the relation $r$, a score for the fact can be recovered as a multilinear product between the embedding vectors of $s$, $r$ and $o$, or through more sophisticated
composition functions~\cite{nickel_2016_review}.


Binary relations in knowledge graphs exhibit various types of patterns: hierarchies and compositions like \tt{FatherOf}, \tt{OlderThan} or \tt{IsPartOf}, with strict/non-strict orders or preorders, and  equivalence relations like \tt{IsSimilarTo}. These characteristics maps to different combinations of the following properties: reflexivity/irreflexivity, symmetry/antisymmetry and transitivity. As described in \citet{Bordes2013}, a relational model should (i) be able to learn all combinations of such properties, and (ii) be linear in both time and memory in order to scale to the size of present-day knowledge graphs, and keep up with their growth. 

A natural way to handle any possible set of relations is to use the classic canonical
polyadic (CP) decomposition \cite{hitchcock-sum-1927}, which yields two different embeddings
for each entity and thus low prediction performances as shown in Section \ref{sec:expe}.
With unique entity embeddings, multilinear products scale well and can naturally handle both symmetry and (ir)-reflexivity of relations, and when combined with an appropriate loss function, dot products can even handle transitivity \cite{bouchard2015approximate}.
However, dealing with antisymmetric---and more generally asymmetric---relations has so far almost always implied superlinear time and space complexity \cite{Nickel2011,socher2013reasoning} (see Section \ref{sec:mat_case}), making models prone to overfitting and not scalable. Finding the best trade-off between expressiveness, generalization and complexity is the keystone of embedding models.

In this work, we argue that the standard dot product between embeddings can be a very effective composition function, provided that one uses the right \emph{representation}: instead of using embeddings containing real numbers, we discuss and demonstrate the capabilities of complex embeddings. When using complex vectors, that is vectors with entries in $\complexSpace$,  the dot product is often called the \emph{Hermitian} (or sesquilinear) dot product, as it involves the conjugate-transpose of one of the two vectors. 
As a consequence, the dot product is not symmetric any more, and facts about one relation can receive different scores depending on the ordering of the entities involved in the fact. 
In summary, complex embeddings naturally represent arbitrary relations while retaining the efficiency of a dot product, that is linearity in both space and time complexity. 


This paper extends a previously published article \cite{trouillon2016}. This extended version adds proofs of existence of the proposed model in both single and multi-relational settings, as well as proofs of the non-uniqueness of the complex embeddings for a given relation. Bounds on the rank of the proposed decomposition are also demonstrated and discussed. The learning algorithm is provided in more details, and more experiments are provided, especially regarding the training time of the models.

The remainder of the paper is organized as follows. We first provide justification and intuition for using complex embeddings in the square matrix case (Section \ref{sec:mat_case}), where there is only a single type of relation between entities, and show the existence of the proposed decomposition for all possible relations. The formulation is then extended to a stacked set of square matrices in a third-order tensor to represent multiple relations (Section \ref{sec:tens_case}). The stochastic gradient descent algorithm used to learn the model is detailed in Section \ref{sec:algo}, where we present an equivalent reformulation of the proposed model that involves only real embeddings. This should help practitioners when implementing our method, without requiring the use of complex numbers in their software implementation. We then describe experiments on large-scale public benchmark knowledge graphs in which we empirically show that this representation leads not only to simpler and faster algorithms, but also gives a systematic accuracy improvement over current state-of-the-art alternatives (Section \ref{sec:expe}). Related work is discussed in Section \ref{sec:rel_work}.

\section{Relations as the Real Parts of Low-Rank Normal Matrices}
\label{sec:mat_case}

We consider in this section a simplified link prediction task with a single relation,
and introduce complex embeddings for low-rank matrix factorization.

We will first discuss the desired properties of embedding models,
show how this problem relates to the spectral theorems, 
and discuss the classes of matrices
these theorems encompass in the real and in the complex case. 
We then propose a new matrix decomposition---to the best of our knowledge---and a 
proof of its existence for all real square matrices.
Finally we discuss the rank of the proposed decomposition.

\subsection{Modeling Relations}
Let $\EntitySpace$ be a set of entities, with $|\EntitySpace|=n$.
The truth of the single relation holding between two entities is represented by a sign value $y_{so}\in\{-1,1\}$, where 1 represents true facts and -1 false facts, $s\in\EntitySpace$ is the subject entity and $o\in\EntitySpace$ is the object entity. The probability for the relation holding true is given by
\begin{equation}
    \proba(y_{so}=1) = \sigma(x_{so})
    \enspace
    \label{observation-model0}
\end{equation}
where $X\in\R^{n\times n}$ is a latent matrix of scores indexed by the subject (rows) and object entities (columns),  $Y$ is a partially-observed sign matrix indexed in identical fashion, and $\sigma$ is a suitable sigmoid function. Throughout this paper we used the logistic inverse link function $\sigma(x) = \frac{1}{1+\mathrm{e}^{-x}}$.

\subsubsection{Handling Both Asymmetry and Unique Entity Embeddings}

In this work we pursue three objectives: finding a generic structure for $X$ that leads to $(i)$ a computationally efficient model, $(ii)$ an expressive enough approximation 
of common relations in real world knowledge graphs, and $(iii)$ good
generalization performances in practice.
Standard matrix factorization approximates $X$ by 
a matrix product $UV\transp$, where $U$ and $V$ are two functionally-independent $n\times K$ matrices, $K$ being the rank of the matrix. Within this formulation it is assumed that entities appearing as subjects are different from entities appearing as objects. In the Netflix challenge \cite{koren_netflix} for example, each row $u_i$ corresponds to the user $i$ and each column $v_j$ corresponds to the movie $j$.
This extensively studied type of model is closely related to the singular value decomposition (SVD) and fits well with the case where the matrix $X$ is rectangular. 

However, in many knowledge graph completion problems, the same entity $i$ can appear as both subject
\emph{or} object and will have two different embedding vectors, $u_i$ and $v_i$, depending on whether it appears as subject or object of a relation. 
It seems natural to learn unique embeddings of entities, as 
initially proposed by \citet{Nickel2011} and \citet{bordes2011learning} 
and since then used systematically in other prominent approaches \cite{bordes2013translating,Yang2015,socher2013reasoning}. 
In the factorization setting, using the same embeddings for left- and right-side factors boils down to a specific case of eigenvalue decomposition: \emph{orthogonal diagonalization}.

\newcommand{\normalSpace}{\mathcal{H}}

\begin{defn}
A real square matrix $\X \in \Mnr$ is orthogonally diagonalizable if it can be
written as $\X= EWE\T$, where $E, W \in \Mnr$, $W$ is diagonal, and $\Eemb$ orthogonal so that $\Eemb\Eemb\T = \Eemb\T \Eemb = I$ where $I$ is the identity matrix.
\end{defn}

The spectral theorem for symmetric matrices tells us that a matrix 
is orthogonally diagonalizable if and only if it is symmetric \cite{cauchy1829}.
It is therefore often used to approximate covariance matrices, kernel functions and distance or similarity matrices.

However as previously stated, this paper is explicitly interested
in problems where matrices---and thus the \mbox{relation patterns} they represent---can 
also be antisymmetric, or even not have any particular symmetry pattern at all (asymmetry). In order to both use a unique embedding for entities and extend the expressiveness to asymmetric relations, researchers have generalised the notion of dot products to \emph{scoring functions}, also known as \emph{composition functions}, that allow more general combinations of embeddings. We briefly recall several examples of scoring functions in Table~\ref{tab:scoring}, as well
as the extension proposed in this paper. 

These models propose different trade-offs between the three essential points:

\begin{itemize}
    \item Expressiveness, which is the ability to represent symmetric, antisymmetric and more generally asymmetric relations.
    \item Scalability, which means keeping linear time and space complexity scoring function.
    \item Generalization, for which having unique entity embeddings is critical.
\end{itemize}

RESCAL \cite{Nickel2011} and NTN \cite{socher2013reasoning} are very expressive,
but their scoring functions have quadratic complexity in the rank of the factorization. 
More recently the \textsc{HolE} model \cite{nickel_2016_holographic}
proposes a solution that has quasi-linear complexity in time and linear space complexity.
\textsc{DistMult} \cite{Yang2015} can be seen as a joint orthogonal 
diagonalization with real embeddings,
hence handling only symmetric relations.
Conversely, \textsc{TransE} \cite{bordes2013translating} handles symmetric relations
to the price of strong constraints on its embeddings.
The canonical-polyadic decomposition (CP) \cite{hitchcock-sum-1927}
generalizes poorly with its different embeddings for entities as subject and as object.

We reconcile expressiveness, scalability and generalization by going back to the realm
of well-studied matrix factorizations, and making use of complex linear algebra,
a scarcely used tool in the machine learning community.

\begin{table*}[t]
    \centering
    
    \resizebox{\columnwidth}{!}{%
    
    \begin{tabular}{|m{5cm}|m{5cm}|m{4cm}|m{1.9cm}|m{1.4cm}|@{}m{0pt}@{}}
        \hline
        \textbf{Model} &
        \textbf{Scoring Function $\phi$} &
        \textbf{Relation Parameters} &
        \textbf{$\mathcal{O}_{time}$}&
        \textbf{$\mathcal{O}_{space}$}
        &\\[5pt]
        \hline
        CP \cite{hitchcock-sum-1927} &  
        $\left<w_r, u_s, v_o\right>$ &
        $w_r \in\Re^{K}$ &
        $\mathcal{O}(K)$& 
        $\mathcal{O}(K)$
        &\\[5pt]
        \hline
        RESCAL \cite{Nickel2011} &
        $e_s^T W_r e_o$ &
        $W_r\in\Re^{K^2}$&
        $\mathcal{O}(K^2)$& 
        $\mathcal{O}(K^2)$
        &\\[5pt]
        \hline
        \textsc{TransE}  \cite{bordes2013translating}&
        $-||(e_s + w_r) - e_o||_p$ &
        $w_r\in\Re^K$ &
        $\mathcal{O}(K)$& 
        $\mathcal{O}(K)$
        &\\[5pt]
        \hline
        NTN  \cite{socher2013reasoning}&
        $u_r\transp f(e_s W_r^{[1..D]}e_o + V_r \begin{bmatrix} e_s 
        \\  e_o \end{bmatrix} + b_r)$ &
        \pbox{20cm}{$W_r\in\Re^{K^2 D}, b_r\in\Re^{K}$\\$V_r \in\Re^{2KD} ,u_r\in\Re^{K} $} &
        $\mathcal{O}(K^2D)$& 
        $\mathcal{O}(K^2D)$
        &\\[20pt]
        \hline
        \textsc{DistMult} \cite{Yang2015}&  
        $\left<w_r, e_s, e_o\right>$ &
        $w_r \in\Re^K$ &
        $\mathcal{O}(K)$& 
        $\mathcal{O}(K)$
        &\\[5pt]
        \hline
        \textsc{HolE} \cite{nickel_2016_holographic}&
        $w_r^T ( \mathcal{F}^{-1}[\overline{\mathcal{F}[e_s]} \odot \mathcal{F}[e_o]]))$
        &
        $w_r \in\Re^K$ &
        $\mathcal{O}(K\log K)$&
        $\mathcal{O}(K)$
        &\\[5pt]
        \hline
        \textsc{ComplEx} (this paper)& 
        $\real(\left<w_r, e_s, \bar{e}_o\right>)$ &
        $w_r\in\complexSpace^{K}$&
        $\mathcal{O}(K)$& 
        $\mathcal{O}(K)$
        &\\[5pt]
        \hline
    \end{tabular}
    }
    \caption{
     Scoring functions of state-of-the-art latent factor models for a given fact $r(s,o)$, along with the representation of their relation parameters, and time and space (memory) complexity. $K$ is the dimensionality of the embeddings. The entity embeddings $e_s$ and $e_o$ of subject $s$ and object $o$ are in $\R^K$ for each model, except for \textsc{ComplEx}, where $e_s,e_o \in \C^K$. $\bar{x}$ is the complex conjugate, and $D$ is an additional latent dimension of the NTN model. 
    $\mathcal{F}$ and $\mathcal{F}^{-1}$ denote respectively the Fourier transform and its inverse,  $\odot$ is the element-wise product between two vectors, $\real(.)$ denotes the real part of a complex vector, and $\left<\cdot,\cdot,\cdot\right>$ denotes the trilinear product.
    }
    \label{tab:scoring}
    \vspace{-.5cm}
\end{table*}



\subsubsection{Decomposition in the Complex Domain}

We introduce a new decomposition of real square matrices using unitary
diagonalization, the generalization of orthogonal diagonalization
to complex matrices. This allows decomposition of \emph{arbitrary} real square matrices
with unique representations of rows and columns.

Let us first recall some notions of complex linear algebra as well as
specific cases of diagonalization of real square matrices, before building
our proposition upon these results.

A complex-valued vector $x\in\C^K$, with $x=\real(x) + i\imag(x)$ is composed 
of a real part
$\real(x)\in\R^K$ and an imaginary part 
$\imag(x)\in\R^K$, where $i$ denotes the square root of $-1$.
The conjugate  $\overline{x}$ of a complex vector inverts the sign
of its imaginary part: $\overline{x}=\real(x) - i\imag(x)$.

Conjugation appears in the usual dot product for complex numbers,
called the \emph{Hermitian} product, or \emph{sesquilinear} form, which is defined as:
\begin{eqnarray}
    \left< u,v \right> &:=& \bar{u}\transp v\notag\\
    &=&\quad\enspace  \real(u)\T \real(v) + \imag(u)\T \imag(v) \notag\\
    && +i(\real(u)\T\imag(v) - \imag(u)\T \real(v) )\notag\,.
    \label{eqn:sesquilinear-dot}
\end{eqnarray}

A simple way to justify the Hermitian product for composing complex vectors is that it provides a valid topological norm in the induced vector space. For example, $\bar{x}\transp x=0$ implies $x=0$ while this is not the case for the bilinear form $x \transp x$ as there are many complex vectors $x$ for which $x\transp x=0$. 

This yields an interesting property of the Hermitian product concerning the
order of the involved vectors: $\left< u, v \right> = \overline{\left< v, u \right>}$, 
meaning that the real part of the product is symmetric,
while the imaginary part is antisymmetric.

For matrices, we shall write $\X^* \in \C^{n \times m}$ for the conjugate-transpose $\X^*= (\overline{\X})\T 
= \overline{\X\T}$. The conjugate transpose is also often written $\X^\dagger$ or $X^{\mathrm{H}}$.

\begin{defn}
A complex square matrix $\X \in \Mnc$ is unitarily diagonalizable if it can be
written as $\X= \Eemb \Wemb \Eemb^*$, where $\Eemb, \Wemb \in \Mnc$, $\Wemb$ is diagonal, and $\Eemb$ is unitary such that $\Eemb\Eemb^* = \Eemb^*\Eemb = I$.
\end{defn}

\begin{defn}
A complex square matrix $\X$ is normal if it commutes with its
conjugate-transpose so that $\X\X^* = \X^*\X$.
\end{defn}



We can now state the spectral theorem for normal matrices.

\begin{thm}[Spectral theorem for normal matrices, \citet{vonneumann1929}]
\label{spectral_thm}

Let $\X$ be a complex square matrix. Then $\X$ is unitarily diagonalizable if and only if
$\X$ is normal.

\end{thm}

It is easy to check that all real symmetric matrices are normal, and have
pure real eigenvectors and eigenvalues.
But the set of purely real normal matrices also includes all
real antisymmetric matrices (useful to model hierarchical
relations such as \tt{IsOlder}), as well as 
all real orthogonal matrices (including permutation matrices), and many other matrices that are useful to represent binary relations, such as assignment matrices which represent bipartite graphs. However, far from all matrices expressed as $\X= \Eemb \Wemb \Eemb^*$ are
purely real, and Equation~(\ref{observation-model0}) requires the scores $X$ to be purely real.

As we only focus on \emph{real} square matrices in this work, let us
summarize all the cases where $\X$ is real square and $\X= \Eemb \Wemb \Eemb^*$
if $X$ is unitarily diagonalizable, where $\Eemb,\Wemb \in \Mnc$, $\Wemb$ is diagonal and $\Eemb$ is unitary:

\begin{itemize}
    \item $\X$ is symmetric if and only if $\X$ is orthogonally diagonalizable and $\Eemb$ and $\Wemb$ are purely real.
    \item $\X$ is normal and non-symmetric if and only if $\X$ is unitarily 
    diagonalizable and $\Eemb$ and $\Wemb$  are \emph{not} both purely real.
    
    \item $\X$ is not normal if and only if $\X$ is not unitarily diagonalizable.
\end{itemize}

We generalize all three cases by showing that, for any 
$\X \in \Mnr$, there exists a unitary diagonalization in the complex domain,
of which the real part equals $\X$:

\begin{equation}
    \X = \real(\Eemb \Wemb \Eemb^* )\,.
    \label{eqn:proj_eigen_dec}
\end{equation}
In other words, the unitary diagonalization is projected onto the real subspace.

\begin{thm}
\label{main_thm}

Suppose $\X \in \Mnr$ is a real square matrix. Then there exists
a normal matrix $\Z \in \Mnc$ such that $ \real(\Z) = \X$.

\end{thm}

\begin{proof}
Let $\Z := \X + i\X\T$. Then
\begin{eqnarray*}
\Z^* = \X\T - i\X = -i(i\X\T + \X) = -i\Z\,,
\end{eqnarray*}
so that 
\begin{eqnarray*}
\Z\Z^* = \Z(-i\Z) = (-i\Z)\Z = \Z^*\Z\,.
\end{eqnarray*}
Therefore $Z$ is normal.
\end{proof}
Note that there also exists a normal matrix $\Z = \X\T + i\X$ such that $ \imag(\Z) = \X$.


Following Theorem~\ref{spectral_thm} and Theorem~\ref{main_thm}, any
real square matrix can be written as the real part of a complex diagonal
matrix through a unitary change of basis.

\begin{cor}
\label{cor_real_diag}

Suppose $\X \in \Mnr$ is a real square matrix. Then there exist $\Eemb,\Wemb \in \Mnc$, where $\Eemb$ is unitary, and $\Wemb$ is diagonal, such that $\X = \real(\Eemb \Wemb \Eemb^* )$.

\end{cor}

\begin{proof}
From Theorem \ref{main_thm}, we can write $\X =\real(\Z) $, where $\Z$ is a normal matrix,
and from Theorem \ref{spectral_thm}, $\Z$ is unitarily diagonalizable.
\end{proof}

Applied to the knowledge graph completion setting, the rows of $E$ here are vectorial representations of the entities corresponding 
to rows and columns of the relation score matrix $\X$. 
The score for the relation holding true between entities $s$ and $o$ is hence
\begin{equation}
x_{so} = \real(e_s\T W \bar{e}_o)
\end{equation}
where $e_s, e_o\in\C^n$ and $W \in\C^{n \times n}$ is diagonal.
For a given entity,
its subject embedding vector is the complex conjugate of its object embedding vector.

To illustrate this difference of expressiveness with respect to real-valued embeddings,
let us consider two complex embeddings $e_s, e_o \in \C$ of dimension 1,
with arbitrary values: $e_s = 1 - 2i$, and $e_o = -3 + i$; as well as
their real-valued, twice-bigger counterparts: $e'_s = \icol{1\\-2}\in \R^2$ and
$e'_o = \icol{-3\\1}\in \R^2$.
In the real-valued case, that corresponds to the \textsc{DistMult} model \cite{Yang2015},
the score is $x_{so} = e\pT_s W' e'_o$.
Figure \ref{fig:complex_vs_real_decomp} represents the heatmaps of the scores
$x_{so}$ and $x_{os}$, as a function of $W \in \C$ in the complex-valued case, 
and as a function of $W' \in \R^2$ diagonal in the real-valued case.
In the real-valued case, that is symmetric in the subject and object entities, 
the scores $x_{so}$ and $x_{os}$ are equal for any value of $W' \in \R^2$ diagonal.
Whereas in the complex-valued case, the variation of $W \in \C$
allows to score $x_{so}$ and $x_{os}$ with any desired pair of values.

\begin{figure}[ht]
	\centering
	\includegraphics[width=0.49\linewidth]{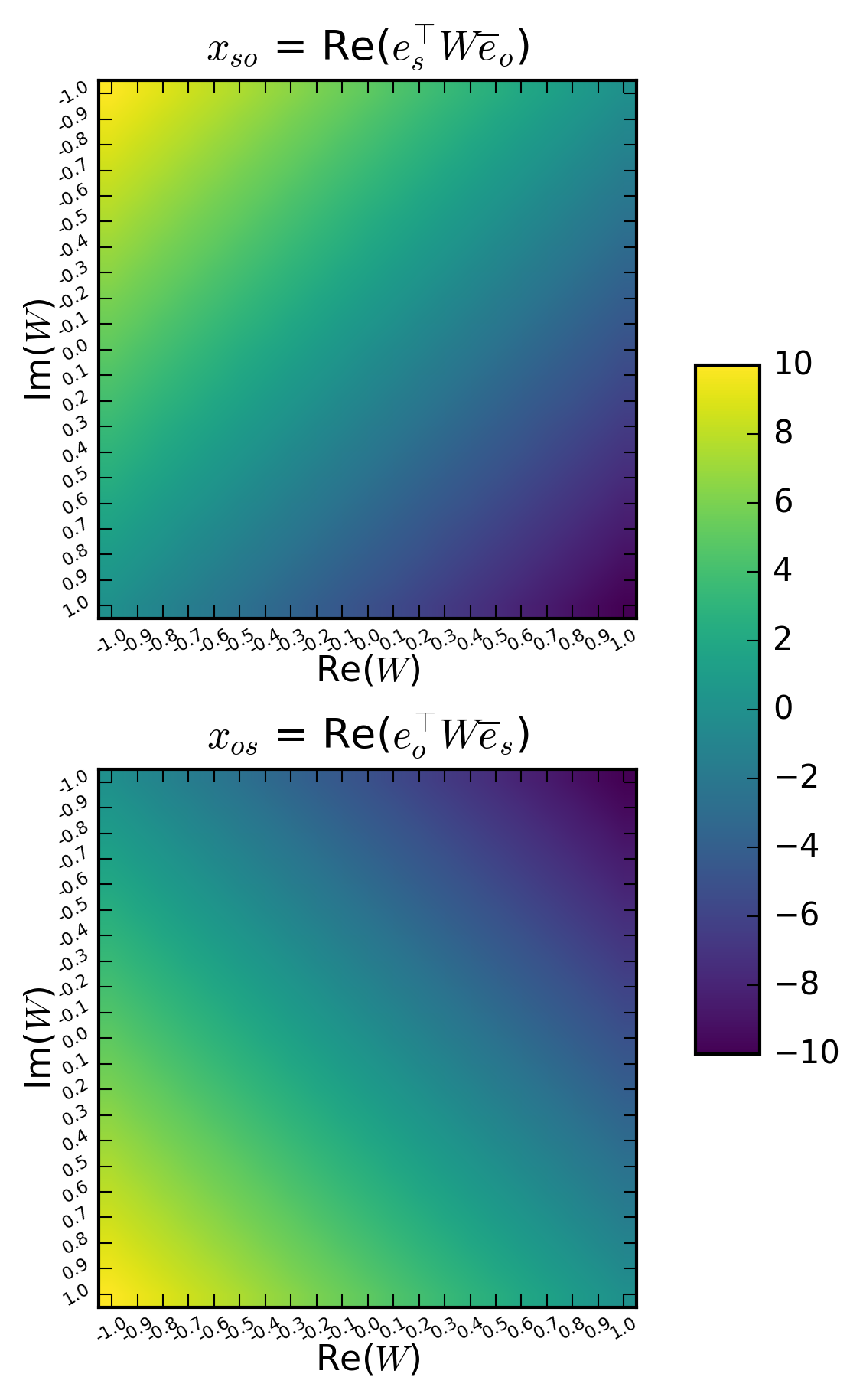}
	\includegraphics[width=0.49\linewidth]{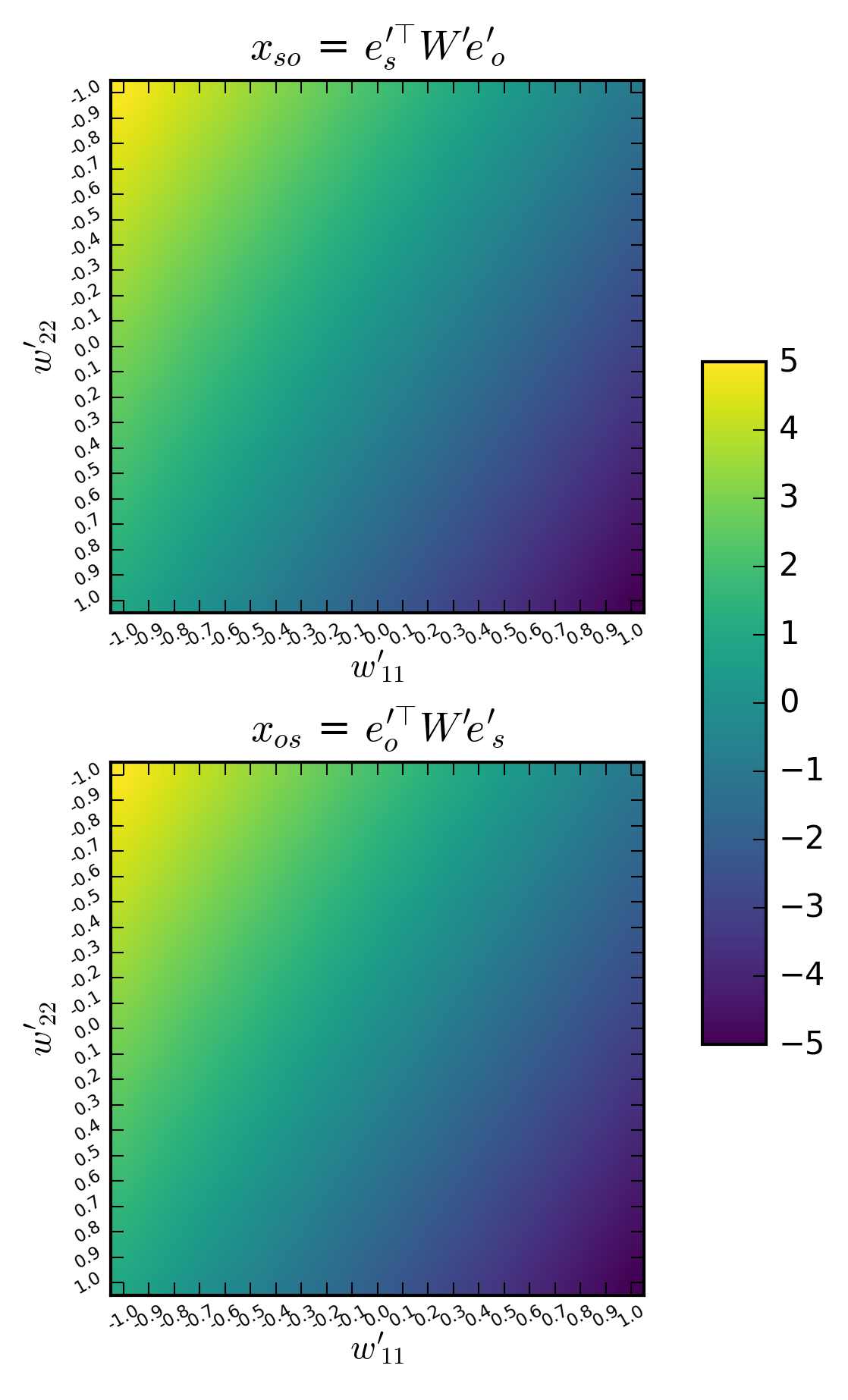}
    \caption{Left: Scores $x_{so}= \real(e_s\T W \bar{e}_o)$ (top) and $x_{os}= \real(e_o\T W \overline{e}_s)$ (bottom) for the proposed complex-valued decomposition, plotted as a function of $W \in \C$, for fixed entity embeddings $e_s = 1 - 2i$, and $e_o = -3 + i$. Right: Scores $x_{so}= e\pT_s W' e'_o$ (top) and $x_{os}= e\pT_o W' e'_s$ (bottom) for the corresponding real-valued decomposition with the same number of free real-valued parameters (\textit{i.e.} in twice the dimension), plotted as a function of $W' \in \R^2$ diagonal, for fixed entity embeddings $e'_s = \usebox{\vecs}$ and $e'_o = \usebox{\veco}$. By varying $W \in \C$, the proposed complex-valued decomposition can attribute any pair of scores to $x_{so}$ and $x_{os}$, whereas $x_{so} = x_{os}$ for all $W' \in \R^2$ with the real-valued decomposition.}
	\label{fig:complex_vs_real_decomp}
\end{figure}

This decomposition however is non-unique, a simple example of this non-uniqueness 
is obtained by adding a purely imaginary constant to the eigenvalues.
Let $X \in \Mnr$, and $X = \real(EWE^*)$ where $E$ is unitary, $W$ is diagonal. Then for any real constant $c \in \R$ we have:
\begin{eqnarray*}
X &=& \real(E(W + icI)E^*)\notag\\
&=& \real(EWE^* + ic EIE^*)\notag\\
&=& \real(EWE^* + icI)\notag\\
&=& \real(EWE^*)\,.\notag\\
\end{eqnarray*}

In general, there are many other possible couples of matrices $E$ and $W$
that preserve the real part of the decomposition.
In practice however this is no synonym of low generalization abilities, 
as many effective matrix and tensor decomposition methods 
used in machine learning lead to non-unique solutions 
\cite{paatero1994positive,Nickel2011}.
In this case also, the learned representations prove useful as 
shown in the experimental section.

\subsection{Low-Rank Decomposition}

Addressing knowledge graph completion with data-driven approaches 
assumes that there is a sufficient regularity in
the observed data to generalize to unobserved facts. When formulated as 
a matrix completion problem, as it is the case in this section, one way
of implementing this hypothesis is to make the assumption
that the matrix has a low rank or approximately low rank.
We first discuss the rank of the proposed decomposition, and then introduce 
the sign-rank and extend the bound developed on the rank to the sign-rank.

\subsubsection{Rank Upper Bound}

First, we recall one definition of the rank of a matrix \cite{horn2012matrix}.

\begin{defn}
The rank of an $m$-by-$n$ complex matrix $\rrank(\X)=\rrank(\X\T)=k$, if 
$\X$ has exactly $k$ linearly independent columns.
\end{defn}

Also note that if $\X$ is diagonalizable so that $\X = \Eemb \Wemb \Eemb^{-1}$ with $\rrank(\X)=k$, then $\Wemb$ has $k$ non-zero diagonal entries for some diagonal $\Wemb$ and some invertible matrix $\Eemb$.
From this it is easy to derive a known additive property of the rank:
\begin{equation}
    \label{eq:rank_add}
    \rrank(B+C) \leq \rrank(B) + \rrank(C)
\end{equation}
where $B,C \in \C^{m \times n}$.

We now show that any rank $k$ real square matrix can be reconstructed 
from a $2k$-dimensional unitary diagonalization.

\begin{cor}
\label{cor_real_rank}
Suppose $\X \in \Mnr$ and $rank(\X) = k$. Then there exist 
$\Eemb \in \C^{n \times 2k}$ such that the columns of $\Eemb$ form an orthonormal basis of $\C^{2k}$,
$\Wemb \in \C^{2k \times 2k}$ is diagonal, and $\X = \real(\Eemb \Wemb \Eemb^* )$.
\end{cor}

\begin{proof}
Consider the complex square matrix $\Z := \X + i\X\T$. We have $\rrank(i\X\T) = \rrank(\X\T) = \rrank(\X) = k$.

From Equation~(\ref{eq:rank_add}), 
$\rrank(\Z) \leq \rrank(\X) + \rrank(i\X\T) = 2k$.


The proof of Theorem \ref{main_thm} shows that $\Z$ is normal. Thus $\Z = \Eemb \Wemb \Eemb^* $ with $\Eemb \in \C^{n \times 2k}$, $\Wemb \in \C^{2k \times 2k}$ where 
the columns of $\Eemb$ form an orthonormal basis of $\C^{2k}$, and $\Wemb$ is diagonal.
\end{proof}

Since $\Eemb$ is not necessarily square, we replace the unitary requirement of Corollary \ref{cor_real_diag}
by the requirement that its columns form an orthonormal basis of its smallest dimension, $2k$.

Also, given that such decomposition always exists in dimension $n$ (Theorem \ref{main_thm}),
this upper bound is not relevant when $\rrank(\X) \geq \frac{n}{2}$.

\subsubsection{Sign-Rank Upper Bound}
Since we encode the truth values of each fact with $\pm 1$, we deal with square \emph{sign matrices}:
$Y \in \{-1,1\}^{n \times n}$. Sign matrices have an alternative rank definition, the \emph{sign-rank}.

\begin{defn}
The sign-rank $\srank(Y)$  of an $m$-by-$n$ sign matrix $Y,$ is the rank of
the $m$-by-$n$ real matrix of least rank that has the same sign-pattern as $Y,$ so that
$$\srank(Y) := \min_{\X\in \R^{m \times n}} \{\rrank(\X)\,|\, \sign(\X) = Y \}\,,$$
where $\sign(\X)_{ij} = \sign(x_{ij})$. 
\end{defn}

We define the \emph{sign function} of $c \in \R$ as 
\begin{equation}
    \sign(c) = \left\{
    \begin{array}{ll}
        1 &\mbox{if } c \geq 0\\
        -1 & \mbox{otherwise}\notag
    \end{array}
\right.
\end{equation}
where the value $c=0$ is here arbitrarily assigned to $1$ to allow zero entries in $X$,
conversely to the stricter usual definition of the sign-rank.

To make generalization possible, we hypothesize that the true matrix $Y$
has a low sign-rank, and thus can be reconstructed by the sign of a low-rank score
matrix $X$. 
The low sign-rank assumption is theoretically justified by the fact that the sign-rank is a natural complexity measure of sign matrices \mbox{\cite{Linial2007}} and is linked to learnability \cite{alon2016sign} and empirically confirmed by the wide success of factorization models~\cite{nickel_2016_review}. 

Using Corollary \ref{cor_real_rank}, we can now show that any square sign matrix of sign-rank $k$ can be reconstructed from a rank $2k$ unitary diagonalization.

\begin{cor}
\label{cor_sign_rank}
Suppose $Y \in \{-1,1\}^{n \times n}$, $\srank(Y)=k$. Then there exists 
$\Eemb \in \C^{n \times 2k}$, $\Wemb \in \C^{2k \times 2k}$ where 
the columns of $\Eemb$ form an orthonormal basis of $\C^{2k}$, and $\Wemb$ is diagonal,
such that $Y = \sign(\real(\Eemb \Wemb \Eemb^* ))$.
\end{cor}

\begin{proof}
By definition, if $\srank(Y)=k$, there exists a real square matrix $\X$ such that
$\rrank(\X)=k$ and $\sign(\X)=Y$.
From Corollary \ref{cor_real_rank}, $\X = \real(\Eemb \Wemb \Eemb^* )$ where $\Eemb \in \C^{n \times 2k}$, $\Wemb \in \C^{2k \times 2k}$ where 
the columns of $\Eemb$ form an orthonormal basis of $\C^{2k}$, and $\Wemb$ is diagonal.
\end{proof}

Previous attempts to approximate the sign-rank in relational learning 
did not use complex numbers. 
They showed the existence of compact factorizations under 
conditions on the sign matrix~\cite{nickel2014reducing}, or only in specific cases~\cite{bouchard2015approximate}.
In contrast, our results show that if a square sign matrix
has sign-rank $k$, then it can be exactly decomposed through a $2k$-dimensional
unitary diagonalization.


Although we can only show the existence of a complex decomposition of rank $2k$ for a matrix
with sign-rank $k$, the sign rank of $Y$ is often \emph{much} lower 
than the rank of $Y$, as we do not know any matrix 
$Y \in \{-1,1\}^{n \times n}$ for
which $\srank(Y) > \sqrt{n}$  \cite{alon2016sign}. 
For example, the  $n \times n$ identity matrix
has rank $n$, but its sign-rank is only 3! By swapping the columns $2j$ and $2j-1$ 
for $j$ in $1,\ldots,\frac{n}{2}$, the identity matrix corresponds to the
relation \texttt{marriedTo}, a relation known to be hard to
factorize over the reals \cite{nickel2014reducing}, since the rank
is invariant by row/column permutations. Yet our model can express it 
at most in rank 6, for any $n$.



Hence, by enforcing a low-rank $K \ll n$ on $E W E^*$, individual relation scores $x_{so}= \real(e_s\transp W \bar{e}_o)$ between entities $s$ and $o$ can be efficiently predicted, as $e_s, e_o\in\C^K$ and $W \in\C^{K \times K}$ is diagonal.

Finding the $K$ that matches the sign-rank of $Y$
corresponds to finding the smallest $K$ that brings the 0--1 loss 
on $X$ to $0$,
as link prediction can be seen as binary classification of the facts.
In practice, and as classically done in machine learning 
to avoid this NP-hard problem,
we use a continuous surrogate of the 0--1 loss, in this case the logistic
loss as described in Section \ref{sec:algo}, and validate
models on different values of $K$, as described in Section
\ref{sec:expe}. 

\subsubsection{Rank Bound Discussion}

Corollaries \ref{cor_real_rank} and \ref{cor_sign_rank} use the aforementioned
subadditive property of the rank to derive the $2k$ upper bound.
Let us give an example for which this bound is strictly
greater than $k$.

Consider the following $2$-by-$2$ sign matrix:
\[
Y=
  \begin{bmatrix}
    -1 & -1\\
     \phantom{-}1 & \phantom{-}1\\
  \end{bmatrix}\,.
\]

Not only is this matrix not normal, but one can also easily check that there is no \emph{real} normal $2$-by-$2$ matrix that has the same sign-pattern as $Y$.
Clearly, $Y$ is a rank $1$ matrix since its columns are linearly dependent,
hence its sign-rank is also $1$. From Corollary \ref{cor_sign_rank}, 
we know that there is a normal matrix whose real part has the same sign-pattern as $Y$,
and whose rank is at most $2$. 


However, there is no rank $1$ unitary diagonalization of which the real part equals $Y$.
Otherwise we could find a 2-by-2 complex matrix $\Z$
such that $\real(z_{11}) < 0$ and $\real(z_{22}) > 0$, 
where $z_{11} = e_1 w \bar{e}_1 = w |e_1|^2$, $z_{22} = e_2 w \bar{e}_2 
= w |e_2|^2$,
$e \in \C^2, w \in \C$. This is obviously unsatisfiable.
This example generalizes to any $n$-by-$n$ square sign matrix that only
has $-1$ on its first row and is hence rank 1, the same argument holds
considering $\real(z_{11}) < 0$ and $\real(z_{nn}) > 0$.

This example shows that the upper bound
on the rank of the unitary diagonalization showed in Corollaries \ref{cor_real_rank} 
and \ref{cor_sign_rank} can be strictly greater than $k$, the rank or sign-rank, 
of the decomposed matrix. However, there might be other examples for which the
addition of an imaginary part could---additionally to making the matrix normal---create 
some linear dependence between the rows/columns and thus decrease the rank of the matrix,
up to a factor of 2.\\

\textbf{We summarize this section in three points:}

\begin{enumerate}
    \item The proposed factorization encompasses all possible score matrices $X$ for a single binary relation.
    \item By construction, the factorization is well suited to represent both symmetric and antisymmetric relations.
    \item Relation patterns can be efficiently approximated
with a low-rank factorization using complex-valued embeddings.
\end{enumerate}

\section{Extension to Multi-Relational Data}
\label{sec:tens_case}

Let us now extend the previous discussion to models with multiple relations.
Let $\RelationSpace$ be the set of relations, with $|\RelationSpace|=m$.
We shall now write $\ScoreTensor \in \R^{m \times n \times n}$ for the score tensor,
$\Relation_r \in \R^{n \times n}$ for the score matrix of the relation $r  \in \RelationSpace$,
and $\ObsTensor \in \{-1,1\}^{m \times n \times n}$ for the partially-observed sign tensor.


Given one relation $r \in \RelationSpace$ and two entities $s,o$ $ \in \EntitySpace$, the probability that the fact \emph{r(s,o)} is true given by:
\begin{equation}
    \proba(y_{rso}=1) = \sigma(x_{rso})=\sigma(\phi(r,s,o;\Theta))
    \label{observation-model}
\end{equation}
where $\phi$ is the scoring function of the model considered and $\Theta$ denotes the model parameters.
We denote the set of all possible facts (or triples) for a knowledge graph by $\TripleSpace = \RelationSpace \times \EntitySpace \times \EntitySpace$.
While the tensor $\ScoreTensor$ as a whole is unknown, we 
assume that we observe a set of true and false triples 
$\OmegaSpace = \{((r,s,o), y_{rso}) \,|\, (r,s,o) \in \TripleSpace_{\OmegaSpace}\}$ 
where $y_{rso} \in \{-1,1\}$
and $\TripleSpace_{\OmegaSpace}\subseteq\TripleSpace$ is the set of observed triples. 
The goal is to find the probabilities of entries 
$y_{r's'o'}$ for a set of targeted unobserved 
triples $\{(r',s',o')\in \TripleSpace \setminus \TripleSpace_{\OmegaSpace}\}$.

Depending on the scoring function $\phi(r,s,o;\Theta)$ used to model the
score tensor $\ScoreTensor$, we obtain different models. Examples of scoring functions are given in Table~\ref{tab:scoring}. 

\subsection{Complex Factorization Extension to Tensors}

The single-relation model is extended by jointly factorizing all 
the square matrices of scores into a $\mathrm{3^{rd}}$-order
tensor $\ScoreTensor \in \R^{m \times n \times n}$, with a different diagonal
matrix $W_r\in \C^{K \times K}$ for each relation $r$, and by sharing the entity embeddings $E \in \C^{n \times K}$ 
across all relations:
\begin{eqnarray}
    \phi(r,s,o;\Theta) &=& \real(\eemb_s\T W_{r} \bar\eemb_o)\notag\\
    &=& \real(\sum_{k=1}^K w_{rk} \eemb_{sk} \bar\eemb_{ok})\notag\\
    \label{eqn:complex-dot}
    &=& \real(\left<w_{r}, \eemb_s, \bar\eemb_o\right>)
    \label{eqn:complex-dot1}
\end{eqnarray}
where $K$ is the rank hyperparameter,
$e_s, e_o \in \C^K$ are the rows in $E$ corresponding to the entities $s$ and $o$, 
$w_r = \mathrm{diag}(W_r) \in \C^K$ is a complex vector,
and $\left<a,b,c\right> := \sum_k a_kb_kc_k$ is the component-wise multilinear 
dot product\footnote{This is not the Hermitian extension of the multilinear 
dot product as there appears to be no standard definition of the Hermitian multilinear product in the linear algebra literature.}.
For this scoring function, the set of parameters $\Theta$ is $\{e_i,w_r \in \C^K,
i \in \setent, r \in \setrel \}$.
This resembles the real part of a complex matrix decomposition as in the single-relation case discussed above. However, we now have a different vector of eigenvalues for every relation.
Expanding the real part of this product gives:
\begin{eqnarray}
    \real(\left<w_{r}, \eemb_s, \bar\eemb_o\right>) &=& \quad \left<\real(w_r),\real(e_s), \real(e_o)\right>\notag\\
    &&+ \left<\real(w_r),\imag(e_s), \imag(e_o)\right> \notag\\
    &&+ \left<\imag(w_r), \real(e_s),\imag(e_o)\right> \notag\\
    &&- \left<\imag(w_r),\imag(e_s),\real(e_o)\right>\,.
    \label{eqn:full_model}
\end{eqnarray}

These equations provide two interesting views of the model:
\begin{itemize}
    \item \emph{Changing the representation}: Equation~(\ref{eqn:complex-dot1}) would correspond to \textsc{DistMult} with real embeddings (see Table \ref{tab:scoring}), but handles asymmetry thanks to the complex conjugate of the object-entity embedding.
   
    \item  \emph{Changing the scoring function}: Equation~(\ref{eqn:full_model}) only involves real vectors corresponding to the real and imaginary parts of the embeddings and relations.%
\end{itemize}

By separating the real and imaginary parts of the relation embedding $\wemb_r$
as shown in Equation~(\ref{eqn:full_model}),
it is apparent that these parts naturally act as weights on each latent dimension: 
$\real(w_r)$ over the real part of $\left< e_o, e_s \right>$ which is symmetric, 
and $\imag(w)$ over the imaginary part
of $\left< e_o, e_s \right>$ which is antisymmetric. 

Indeed, the decomposition of each score matrix $\Relation_r$ 
for each $r \in \RelationSpace$ can be written as
the sum of a symmetric matrix and an antisymmetric matrix.
To see this, let us rewrite the decomposition of each
score matrix $\Relation_r$ in matrix notation.
We write the real part of matrices with primes $\Eemb' = \real(\Eemb)$ and
imaginary parts with double primes $\Eemb'' = \imag(\Eemb)$:
\begin{eqnarray}
    \Relation_r &=& \real( \Eemb \Wemb_r \Eemb^* )\notag\\
                &=& \real( (\Eemb'+ i \Eemb'') (\Wemb'_r+ i \Wemb''_r) (\Eemb'- i \Eemb'')\transp )\notag\\
                &=& (\Eemb' \Wemb_r' \Eemb'^\transp + \Eemb'' \Wemb_r' \Eemb''^\transp) 
                + (\Eemb' \Wemb_r'' \Eemb''^\transp - \Eemb'' \Wemb_r'' \Eemb'^\transp)\,.
    \label{eq:as_symm_antisym_sum}
\end{eqnarray}
It is trivial to check that the matrix $\Eemb' \Wemb_r' \Eemb'^\transp + \Eemb'' \Wemb_r' \Eemb''^\transp$ is symmetric and that the matrix $\Eemb' \Wemb_r'' \Eemb''^\transp - \Eemb'' \Wemb_r'' \Eemb'^\transp$ is antisymmetric.
Hence this model is well suited to model jointly symmetric and antisymmetric relations between pairs of entities, while still using the same entity representations for subjects and objects.
When learning, it simply needs to 
collapse $\Wemb''_r = \imag(\Wemb_r)$ to zero for symmetric relations $r \in \RelationSpace$, and $\Wemb'_r = \real(\Wemb_r)$
to zero for antisymmetric relations $r \in \RelationSpace$,
as $\Relation_r$
is indeed symmetric when $\Wemb_r$ is purely real, 
and antisymmetric when $\Wemb_r$ is purely imaginary.

From a geometrical point of view, 
each relation embedding $w_r$ is an anisotropic
scaling of the basis defined by the entity embeddings $E$, 
followed by a projection onto the real subspace.

\subsection{Existence of the Tensor Factorization}

Let us first discuss the existence of the multi-relational
model where the rank of the decomposition $K \leq n$,
which relates to simultaneous unitary decomposition.

\begin{defn}
    A family of matrices $X_1,\ldots,X_m \in \Mnc$ is simultaneously unitarily diagonalizable,
    if there is a single unitary matrix $E \in \Mnc$, such that $X_i = EW_iE^*$  for all $i$ in $1,\ldots,m$, where $W_i \in \Mnc$ are diagonal.
\end{defn}

\begin{defn}
    A family of normal matrices $X_1,\ldots,X_m \in \Mnc$
    is a commuting family of normal matrices, if 
    $X_i X_j^* = X_i^* X_j$, for all $i,j$ in $1,\ldots,m$.
\end{defn}

\begin{thm}[see \citet{horn2012matrix}]
    \label{thm:commuting}
    Suppose $\mathcal{F}$ is  the family of matrices $X_1,$ $\ldots$ $,X_m \in \Mnc$.
    Then $\mathcal{F}$ is a commuting family of normal matrices if and only if
    $\mathcal{F}$ is simultaneously unitarily diagonalizable.
\end{thm}

To apply Theorem \ref{thm:commuting} to the proposed factorization, 
we would have to make the hypothesis
that the relation score matrices $\Relation_r$ are a commuting family, which is too strong
a hypothesis. Actually, the model is slightly different since we take only the real part
of the tensor factorization. In the single-relation case, taking only
the real part of the decomposition rids us of the normality requirement 
of Theorem \ref{spectral_thm} for the decomposition to exist, as shown in Theorem \ref{main_thm}.

In the multiple-relation case, it is an open question whether 
taking the real part of the simultaneous unitary diagonalization will enable 
us to decompose families of arbitrary real square 
matrices---that is with a single unitary matrix $E$ that has \emph{at most} $n$ columns.
Though it seems unlikely, we could not find a counter-example yet.

However, by letting the rank of the tensor factorization $K$ to be greater than $n$,
we can show that the proposed tensor decomposition exists for families of arbitrary real square 
matrices, by simply concatenating the decomposition of Theorem \ref{main_thm} 
of each real square matrix $X_i$.

\begin{thm}
    \label{thm:main_thm_tensors}
    Suppose $X_1,\ldots,X_m \in \Mnr$. Then there exists $E \in \C^{n \times nm}$ and 
    $W_i \in \C^{nm \times nm}$ are diagonal, 
    such that $X_i = \real(E W_i E^*)$ for all $i$ in $1,\ldots,m$.
\end{thm}
\begin{proof}
From Theorem \ref{main_thm} we have $X_i = \real(E_i W_i E_i^*)$, 
where $W_i \in\C^{n \times n}$ is diagonal,
and each $E_i \in\C^{n \times n}$ is unitary for all $i$ in $1,\ldots,m$.

Let $E = \left[ E_1 \ldots E_m \right]$, and
\begin{align}
    \Lambda_i &= \begin{bmatrix}
           \mathbf{0}^{((i-1)n) \times ((i-1)n)} & &\\
           & W_i&& \\
           & & \mathbf{0}^{((m-i)n) \times ((m-i)n)}\notag
         \end{bmatrix}
\end{align}
where $\mathbf{0}^{l \times l}$ the zero $l \times l$ matrix.
Therefore $X_i = \real(E \Lambda_i E^*)$ for all $i$ in $1,\ldots,m$.
\end{proof}

By construction, the rank of the decomposition is at most $nm$. 
When $m \leq n$, this bound actually matches the general upper bound on 
the rank of the canonical polyadic (CP) decomposition
\cite{hitchcock-sum-1927,kruskal1989rank}. Since $m$ corresponds to the number of relations and $n$
to the number of entities, $m$ is always smaller than $n$ in real world knowledge graphs,
hence the bound holds in practice.

Though when it comes to relational learning, we might expect
the actual rank to be much lower than $nm$ for two reasons. The first one, as discussed
above, is that we are dealing with sign tensors, hence the rank
of the matrices $\Relation_r$ need only match the sign-rank of the
partially-observed matrices $Y_r$.
The second one is that the matrices are related to each other,
as they all represent the same entities in different relations,
and thus benefit from sharing latent dimensions. As opposed to the construction
exposed in the proof of Theorem \ref{thm:main_thm_tensors}, where
other relations dimensions are canceled out. In practice,
the rank needed to generalize well is indeed much lower than $nm$ as we
show experimentally in Figure \ref{fig:mrr_vs_rank}. 

Also, note that
with the construction of the proof of Theorem \ref{thm:main_thm_tensors},
the matrix $E = \left[ E_1 \ldots E_m \right]$
is not unitary any more. 
However the unitary constraints in the matrix case serve only
the proof of existence, which is just one solution among the
infinite ones of same rank. 
In practice, imposing orthonormality is essentially a numerical 
commodity for the decomposition of dense matrices,
through iterative methods for example \cite{saad1992numerical}.
When it comes to matrix and tensor completion, and
thus generalisation, imposing such constraints is more of
a numerical hassle than anything else, especially for
gradient methods. As there is no apparent link
between orthonormality and generalisation properties,
we did not impose these constraints when learning this model 
in the following experiments.

\section{Algorithm}
\label{sec:algo}

Algorithm \ref{SGDC} describes stochastic gradient descent (SGD) to learn the proposed 
multi-relational model with the AdaGrad learning-rate updates \cite{duchi2011adaptive}. 
We refer to the proposed model as \textsc{ComplEx}, for Complex Embeddings.
We expose a version of the algorithm that uses only real-valued vectors, in order 
to facilitate its implementation. To do so, we use separate real-valued representations of the
real and imaginary parts of the embeddings.

These real and imaginary part vectors are 
initialized with vectors having a zero-mean normal distribution with unit variance.
If the training set $\Omega$ contains only positive triples, negatives are generated for
each batch using the \emph{local closed-world assumption} as in \citet{bordes2013translating}. 
That is, for each triple, we randomly change either the subject or the object, 
to form a negative example.
In this case the parameter $\eta>0$ sets the number of negative triples
to generate for each positive triple. Collision with positive triples in $\Omega$
is not checked, as it occurs rarely in real world knowledge graphs as they are largely sparse,
and may also be computationally expensive.

Squared gradients are accumulated to compute AdaGrad learning rates, then gradients are updated.
Every $s$ iterations, the parameters $\Theta$ are evaluated over the evaluation set $\Omega_v$
({\it evaluate\_AP\_or\_MRR}$(\Omega_v;\Theta)$ function 
in Algorithm \ref{SGDC}). If the data set contains both positive and negative examples,
average precision (AP) is used to evaluate the model. If the data set 
contains only positives, then mean reciprocal rank (MRR) is used as average precision
cannot be computed without true negatives. The optimization process is stopped when 
the measure considered decreases compared to the last evaluation (early stopping).

Bern($p$) is the Bernoulli distribution, the {\it one\_random\_sample}$(\EntitySpace)$ function
sample uniformly one entity in the set of all entities $\EntitySpace$,
and the {\it sample\_batch\_of\_size\_b}$(\Omega,b)$
function sample $b$ true and false triples uniformly at random
from the training set $\Omega$.

For a given embedding size $\rank$, let us rewrite Equation~(\ref{eqn:full_model}), by denoting
the real part of embeddings with primes and the imaginary part with double primes: 
$e'_i = \real(e_i)$, $e''_i = \imag(e_i)$, $w'_r = \real(w_r)$, $w''_r = \imag(w_r)$.
The set of parameters is $\Theta=\{e'_i,e''_i,w'_r,w''_r \in \R^K,
i \in \setent, r \in \setrel \}$,
and the scoring function involves only real vectors:
\begin{alignat}{999}
\phi(r,s,o;\Theta) = \quad &\left<w'_r,e'_s, e'_o\right> &&+ \left<w'_r,e''_s, e''_o\right> \notag\\
    + &\left<w''_r, e'_s, e''_o\right> &&- \left<w''_r,e''_s, e'_o\right>
\end{alignat}
where each entity and each relation has two real embeddings.

Gradients are now easy to write:
\begin{alignat*}{999}
\label{gradients}
    \grad_{e'_s} \phi(r,s,o;\Theta) &= (w'_r &&\odot e'_o) &&+ (w''_r &&\odot e''_o),\\
    \grad_{e''_s} \phi(r,s,o;\Theta) &= (w'_r &&\odot e''_o) &&- (w''_r &&\odot e'_o),\\
    \grad_{e'_o} \phi(r,s,o;\Theta) &= (w'_r &&\odot e'_s) &&- (w''_r &&\odot e''_s),\\
    \grad_{e''_o} \phi(r,s,o;\Theta) &= (w'_r &&\odot e''_s) &&+ (w''_r &&\odot e'_s),\\
    \grad_{w'_r} \phi(r,s,o;\Theta) &= (e'_s &&\odot e'_o) &&+ (e''_s &&\odot e''_o),\\
    \grad_{w''_r} \phi(r,s,o;\Theta) &= (e'_s &&\odot e''_o) &&- (e''_s &&\odot e'_o),
\end{alignat*}
where $\odot$ is the element-wise (Hadamard) product.

We optimized the negative log-likelihood of the logistic model described in Equation~(\ref{observation-model}) with $L^2$ regularization on the parameters $\Theta$:
\begin{eqnarray}
    \gamma(\Omega;\Theta) &=&
    \sum_{((r,s,o),y) \in \Omega} \log( 1 + \exp(-y \phi(r,s,o;\Theta))) + \lambda ||\Theta||^2_2\
    \label{eq:objective}
\end{eqnarray}
where $\lambda \in \R_+$ is the regularization parameter.


To handle regularization, note that using separate representations
for the real and imaginary parts does not change anything as
the squared $L^2$-norm of a complex vector $v=v'+iv''$
is the sum of the squared modulus of each entry:
\begin{eqnarray}
||v||^2_2 &=& \sum_j \sqrt{v_j^{\prime2} + v_j^{\prime\prime2}}^2\notag\\
&=& \sum_j v_j^{\prime2} + \sum_j  v_j^{\prime\prime2}\notag\\
&=& ||v'||^2_2 +  ||v''||^2_2\notag\,,
\end{eqnarray}
which is actually the sum of the $L^2$-norms of the vectors of the real and imaginary parts.

We can finally write the gradient of $\gamma$ with respect to a \emph{real} embedding $v$ for
one triple $(r,s,o)$ and its truth value $y$:
\begin{eqnarray}
    \grad_v \gamma(\{((r,s,o),y)\};\Theta) &=& -y \sigma(-y \phi(r,s,o;\Theta)) \grad_v \phi(r,s,o;\Theta) + 2\lambda v\,.
\end{eqnarray}

\begin{algorithm}[H]
\caption{Stochastic gradient descent with AdaGrad for the \textsc{ComplEx} model}
\label{SGDC}
\begin{algorithmic}
\INPUT Training set $\Omega$, validation set $\Omega_v$, learning rate $\alpha\in \mathbb{R}_{++}$, rank $K\in \mathbb{Z}_{++}$, $L^2$ regularization factor $\lambda\in \mathbb{R}_{+}$, negative ratio $\eta\in \mathbb{Z}_{++}$, batch size $b\in \mathbb{Z}_{++}$, maximum iteration $m\in \mathbb{Z}_{++}$, validate every $s\in \mathbb{Z}_{++}$ iterations, AdaGrad regularizer $\epsilon = 10^{-8}$.
\OUTPUT Embeddings $e', e'', w', w''$.
\STATE $e'_i\sim \mathcal{N}(\mathbf{0}^k, I^{k \times k})$ , $e''_i \sim \mathcal{N}(\mathbf{0}^k, I^{k \times k})$ for each $i \in \mathcal{E}$
\STATE $w'_i \sim \mathcal{N}(\mathbf{0}^k, I^{k \times k})$, $w''_i \sim \mathcal{N}(\mathbf{0}^k, I^{k \times k})$ for each $i \in \mathcal{R}$

\STATE $g_{e'_i} \gets \mathbf{0}^k$ , $g_{e''_i} \gets \mathbf{0}^k $ for each $i \in \mathcal{E}$
\STATE $g_{w'_i} \gets \mathbf{0}^k$ , $g_{w''_i} \gets \mathbf{0}^k $ for each $i \in \mathcal{R}$
\STATE $previous\_score \gets 0$
\FOR{$i=1,\ldots,m$}
    \FOR {$j=1,\ldots,|\Omega|/b$}
        \STATE $\Omega_b \gets$ {\it sample\_batch\_of\_size\_b}$(\Omega,b)$
        \STATE // Negative sampling:
        \STATE $\Omega_n \gets \{ \emptyset \}$
        \FOR {$((r,s,o),y)$ in $\Omega_b$}
            \FOR{$l=1,\ldots,\eta$}
                \STATE $e \gets$ {\it one\_random\_sample}$(\EntitySpace)$
                \IF{Bern$(0.5) >$ 0.5}
                    \STATE $\Omega_n \gets \Omega_n \cup \{((r,e,o),-1)\}$ 
                \ELSE
                    \STATE $\Omega_n \gets \Omega_n \cup \{((r,s,e),-1)\}$ 
                \ENDIF
            \ENDFOR
        \ENDFOR
        \STATE $\Omega_b \gets \Omega_b \cup \Omega_n$
        \FOR {$((r,s,o),y)$ in $\Omega_b$}
            \FOR{$v$ in $\Theta$}
                \STATE // AdaGrad updates:
                \STATE $g_v \gets g_v + (\grad_v \gamma(\{((r,s,o),y)\};\Theta))^2$
                \STATE // Gradient updates:
                \STATE $v \gets v - \frac{\alpha}{g_v + \epsilon} \grad_v \gamma(\{((r,s,o),y)\};\Theta)$
            \ENDFOR
        \ENDFOR
    \ENDFOR
    \STATE // Early stopping
    \IF{$i \mod s = 0$}
        \STATE $current\_score \gets$ {\it evaluate\_AP\_or\_MRR}$(\Omega_v;\Theta)$ 
            \IF{$current\_score \leq previous\_score$}
                \STATE \textbf{break} 
            \ENDIF
            \STATE $previous\_score \gets current\_score$
    \ENDIF
\ENDFOR
\STATE \textbf{return} $\Theta$
\end{algorithmic}
\end{algorithm}

\section{Experiments}
\label{sec:expe}

We evaluated the method proposed in this paper on both synthetic and real data sets. 
The synthetic data set contains both symmetric and antisymmetric
relations, whereas the real data sets are standard link prediction benchmarks
based on real knowledge graphs. 

We compared \textsc{ComplEx} to state-of-the-art models, namely \textsc{\textsc{TransE}} \cite{bordes2013translating},
\textsc{DistMult} \cite{Yang2015}, RESCAL \cite{Nickel2011} and also to the canonical polyadic decomposition (CP)
\cite{hitchcock-sum-1927}, to emphasize empirically the importance of learning 
unique embeddings for entities. 
For experimental fairness, we reimplemented these models within the same framework as the \textsc{ComplEx} model, using a Theano-based SGD implementation\footnote{https://github.com/lmjohns3/downhill} \cite{theano}. 

For the \textsc{TransE} model, results were obtained with its 
original max-margin loss, as
it turned out to yield better results for this model only. 
To use this max-margin loss on data sets with observed
negatives (Sections \ref{sec:synth_task} and \ref{sec:kinships_umls}), positive
triples were replicated when necessary to match the number of negative triples, 
as described in \citet{garcia2016combining}.
All other models are trained with the negative log-likelihood of the logistic model (Equation~(\ref{eq:objective})). In all the following experiments we used a maximum number of iterations $m=1000$,
a batch size $b= \frac{|\Omega|}{100}$, and validated the models for early stopping 
every $s=50$ iterations.

\subsection{Synthetic Task}
\label{sec:synth_task}

To assess our claim that \textsc{ComplEx} can accurately model jointly symmetry and antisymmetry, we randomly generated a knowledge graph of two relations and 30 entities. One relation is entirely symmetric, while the other is completely antisymmetric. This data set corresponds to a $2 \times 30 \times 30$ tensor. Figure \ref{fig:symmetry_example} shows a part of this randomly generated tensor, with a symmetric slice and an antisymmetric slice, decomposed into training, validation and test sets. To ensure that all test values are predictable, the upper triangular parts of the matrices are always kept in the training set, and the diagonals are unobserved. We conducted a 5-fold cross-validation on the lower-triangular matrices, using the upper-triangular parts plus 3 folds for training, one fold for validation and one fold for testing. Each training set contains 1392 observed triples, whereas validation and test sets contain 174 triples each. 

Figure \ref{fig:exp_sym_antisym} shows the best cross-validated average precision (area under the precision-recall curve) for different factorization models of ranks ranging up to 50. The regularization parameter $\lambda$ is validated in $\{$0.1, 0.03, 0.01, 0.003,0.001, 0.0003, 0.00001, 0.0$\}$ and the learning rate $\alpha$ was initialized to 0.5.

As expected, \textsc{DistMult} \cite{Yang2015} is not able to model antisymmetry and only predicts the symmetric relations correctly. Although \textsc{TransE} \cite{bordes2013translating} is not a symmetric model, it performs poorly in practice, particularly on the antisymmetric relation. 
RESCAL \cite{Nickel2011}, with its large number of parameters, quickly overfits as the rank grows. Canonical Polyadic (CP) decomposition \cite{hitchcock-sum-1927} fails on both relations as it has to push symmetric and antisymmetric patterns through the entity embeddings. Surprisingly, only \textsc{ComplEx} succeeds even on such simple data.

\begin{figure}[ht]
	\centering
	\includegraphics[width=0.28\linewidth]{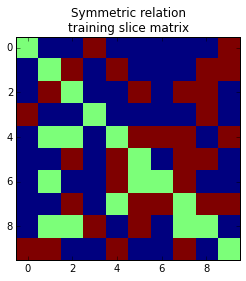}
	\includegraphics[width=0.28\linewidth]{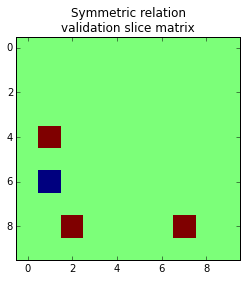}
	\includegraphics[width=0.28\linewidth]{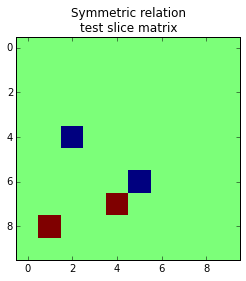}
	
	\includegraphics[width=0.28\linewidth]{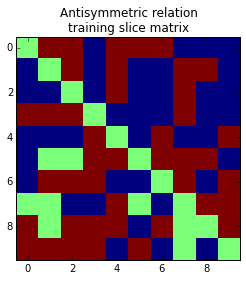}
	\includegraphics[width=0.28\linewidth]{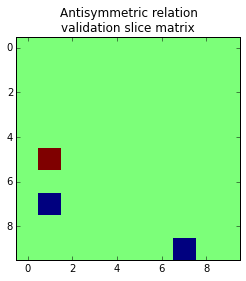}
	\includegraphics[width=0.28\linewidth]{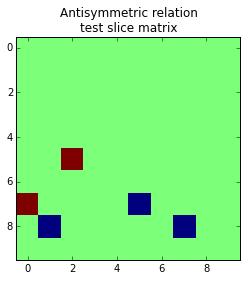}

    \caption{Parts of the training, validation and test sets of the generated experiment with one symmetric and one antisymmetric relation. Red pixels are positive triples, blue are negatives, and green missing ones. Top: Plots of the symmetric slice (relation) for the 10 first entities. Bottom: Plots of the antisymmetric slice for the 10 first entities.}
	\label{fig:symmetry_example}
\end{figure}



\begin{figure*}[ht]{{\extracolsep{8pt}}}
	\centering
	\includegraphics[width=0.49\textwidth]{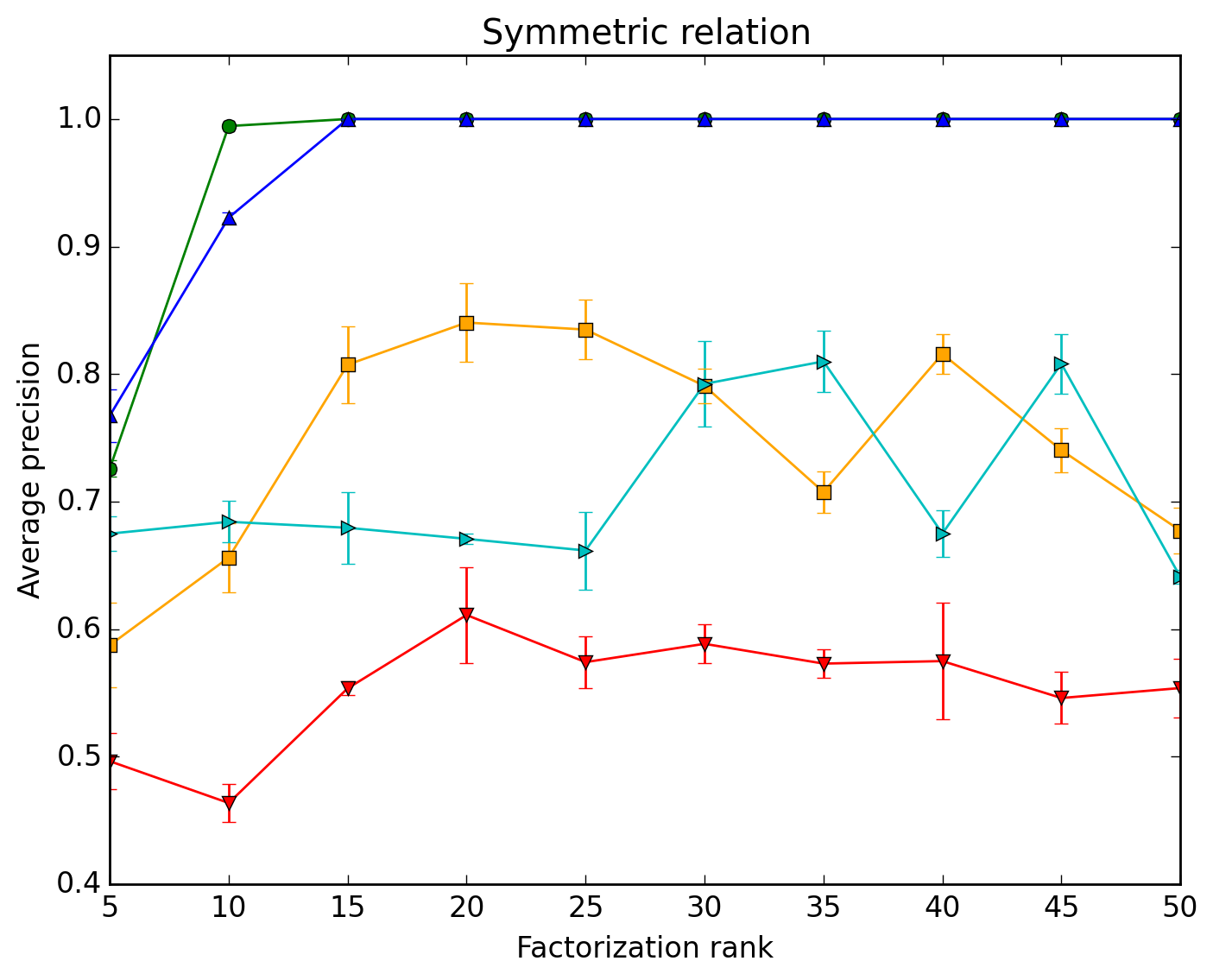}
	\includegraphics[width=0.49\textwidth]{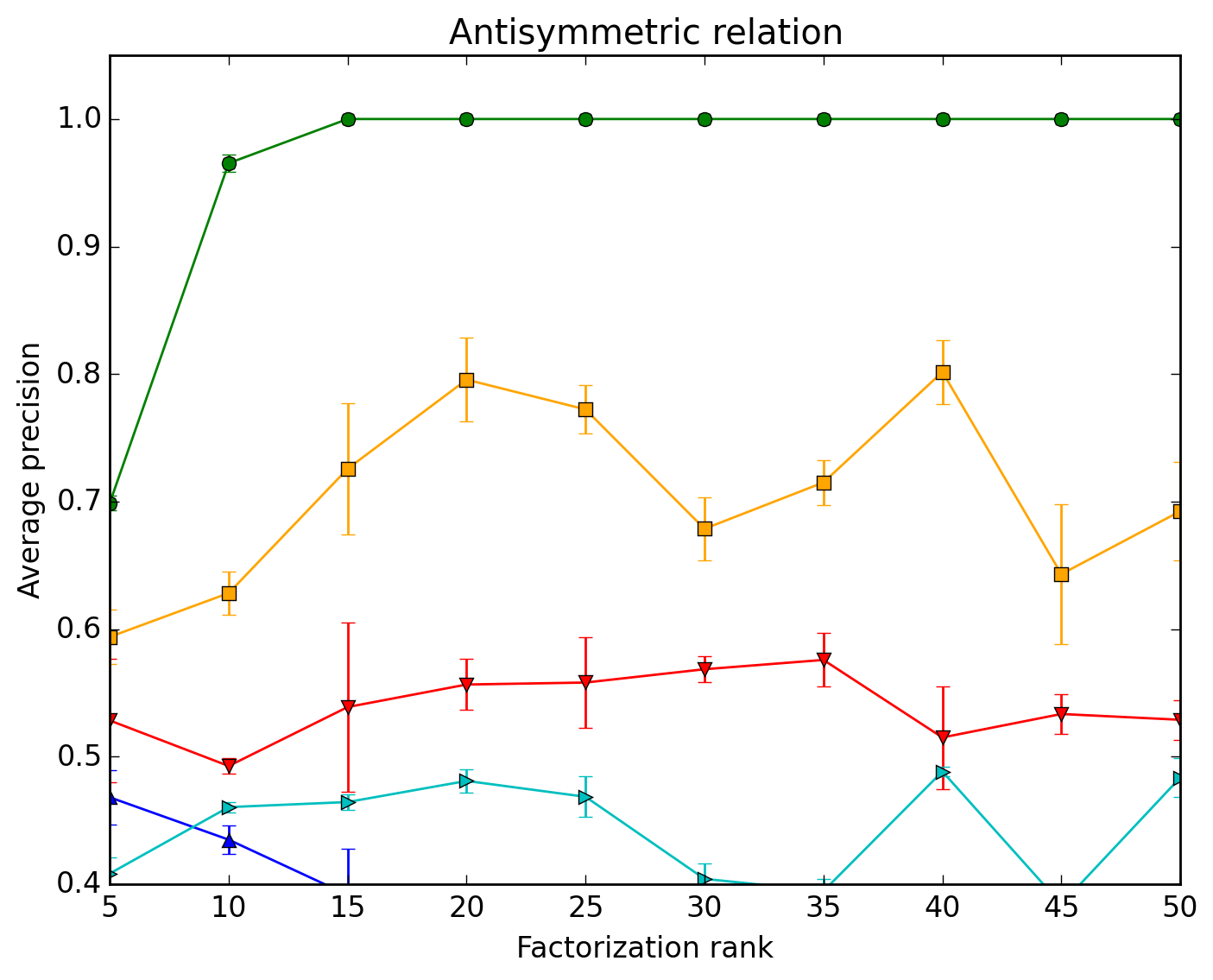} 
	\includegraphics[width=0.65\textwidth]{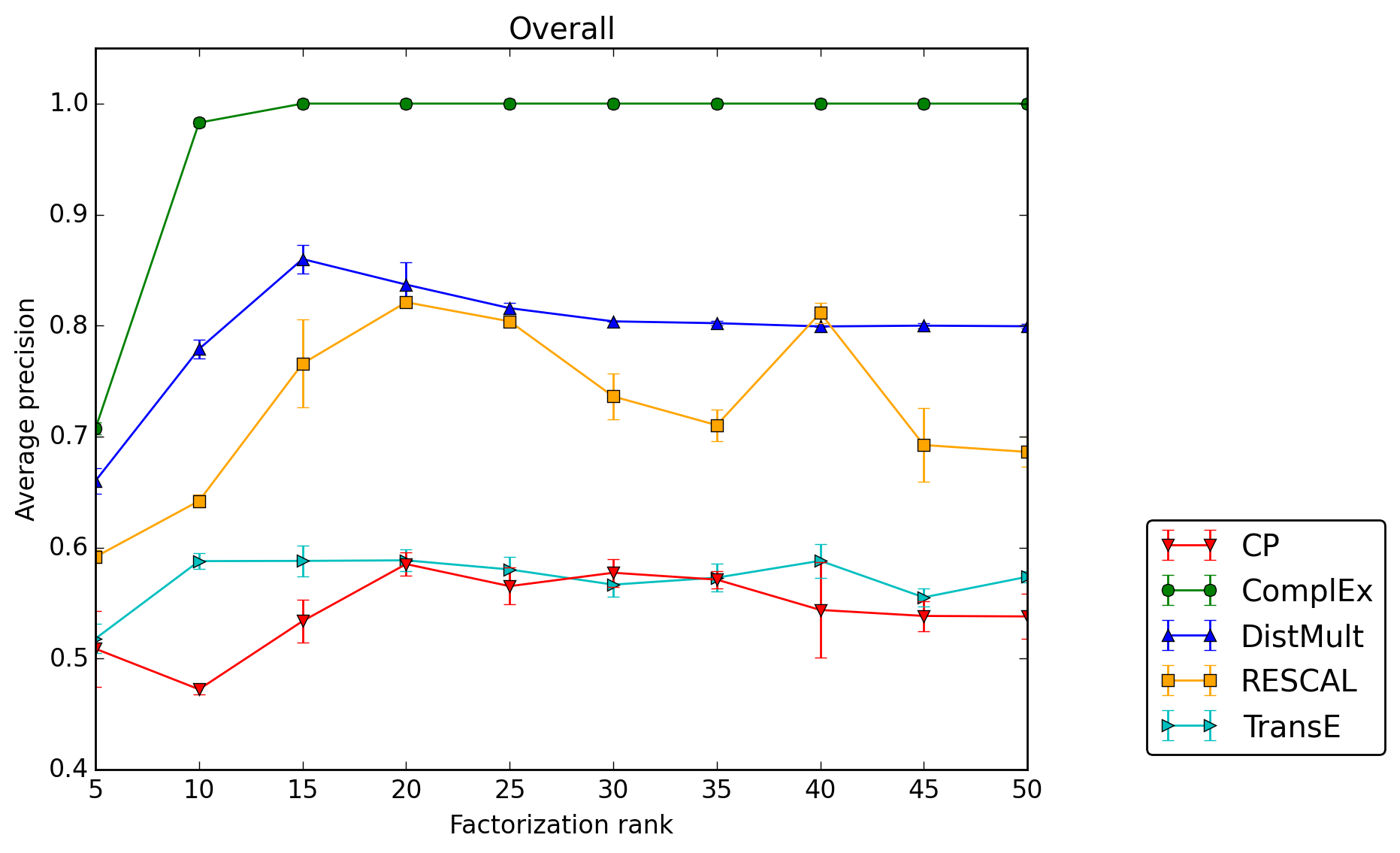}
	\caption{Average precision (AP) for each factorization rank ranging from 5 to 50 for different state-of-the-art models on the synthetic task. Learning is performed jointly on the symmetric relation and on the antisymmetric relation. Top-left: AP over the symmetric relation only. Top-right: AP over the antisymmetric relation only. Bottom: Overall AP.}
	\label{fig:exp_sym_antisym}
\end{figure*}

\subsection{Real Fully-Observed Data Sets: Kinships and UMLS}
\label{sec:kinships_umls}

We then compare all models on two fully observed data sets, that
contain both positive and negative triples, also called the \emph{closed-world
assumption}.
The Kinships data set \cite{denham1973detection}
describes the 26 different kinship relations of the Alyawarra tribe in Australia, among 104 individuals.
The unified medical language system (UMLS) data set \cite{mccray2003upper} represents 135 medical concepts and diseases, linked
by 49 relations describing their interactions. 
Metadata for the two data sets is summarized in Table
\ref{tab:kinship_and_umls_meta}.

\begin{table}[ht]
    \centering
    \begin{tabular}{l|rrr}
        Data set & $|\EntitySpace|$ & $|\RelationSpace|$ & Total number of triples\\ \hline
        Kinships & 104 & 26 & 281,216 \\
        UMLS & 135 & 49 & 893,025 \\
    \end{tabular}
    \caption{Number of entities $|\EntitySpace|$, relations $|\RelationSpace|$, and observed triples (all are observed) for the Kinships and UMLS data sets.}
    \label{tab:kinship_and_umls_meta}
\end{table}

We performed a 10-fold cross-validation, keeping 8 for training, one for validation and one for testing.
Figure \ref{fig:exp_kinships_umls} shows the best cross-validated average precision for ranks ranging up to 50, and error bars show the standard deviation over the 10 runs.
The regularization parameter $\lambda$ is validated in $\{$0.1, 0.03, 0.01, 0.003, 0.001, 0.0003, 0.00001, 0.0$\}$ and the learning rate $\alpha$ was initialized to 0.5.

\begin{figure}[t]
	\centering
	\includegraphics[width=0.75\textwidth]{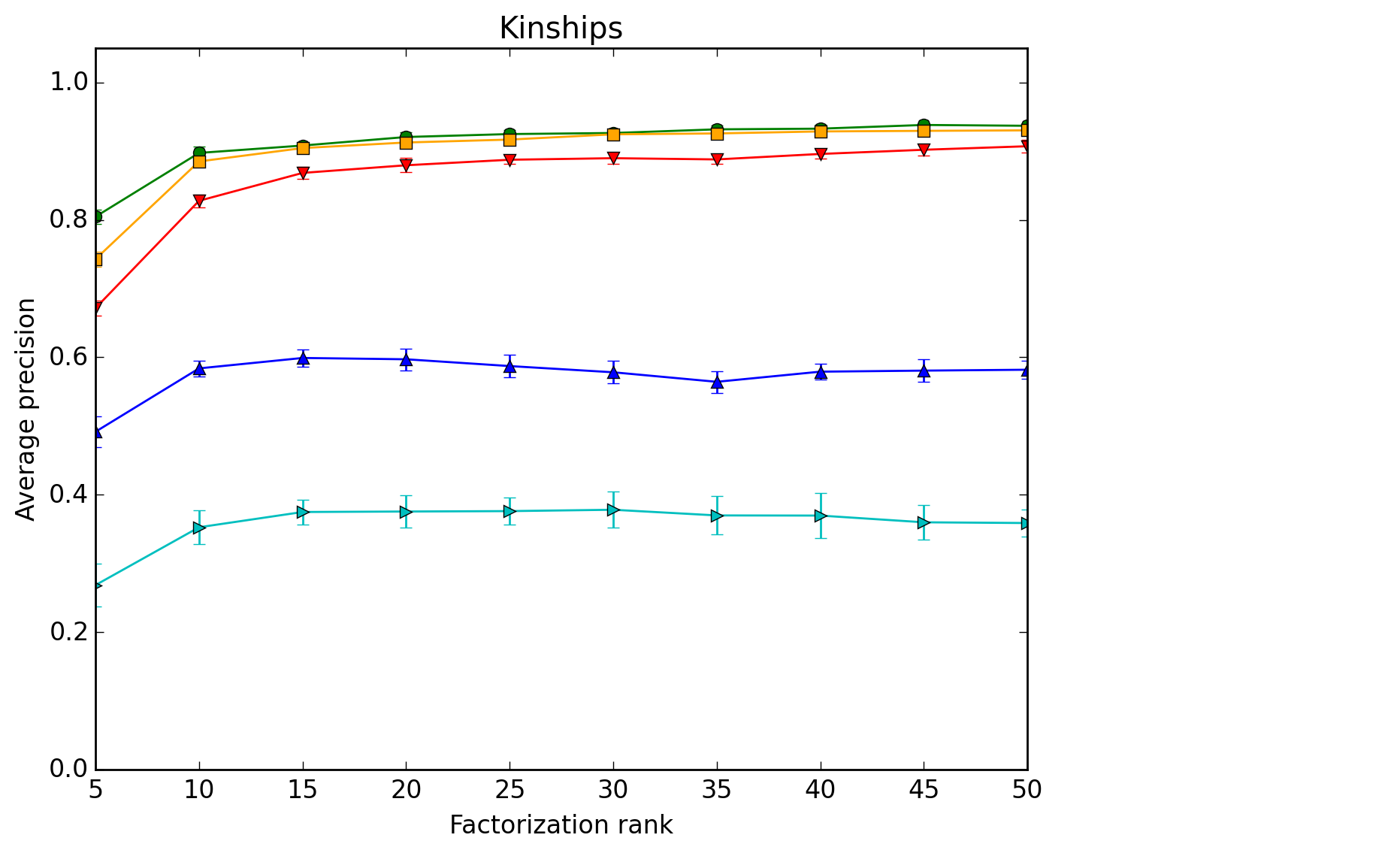}
	\includegraphics[width=0.75\textwidth]{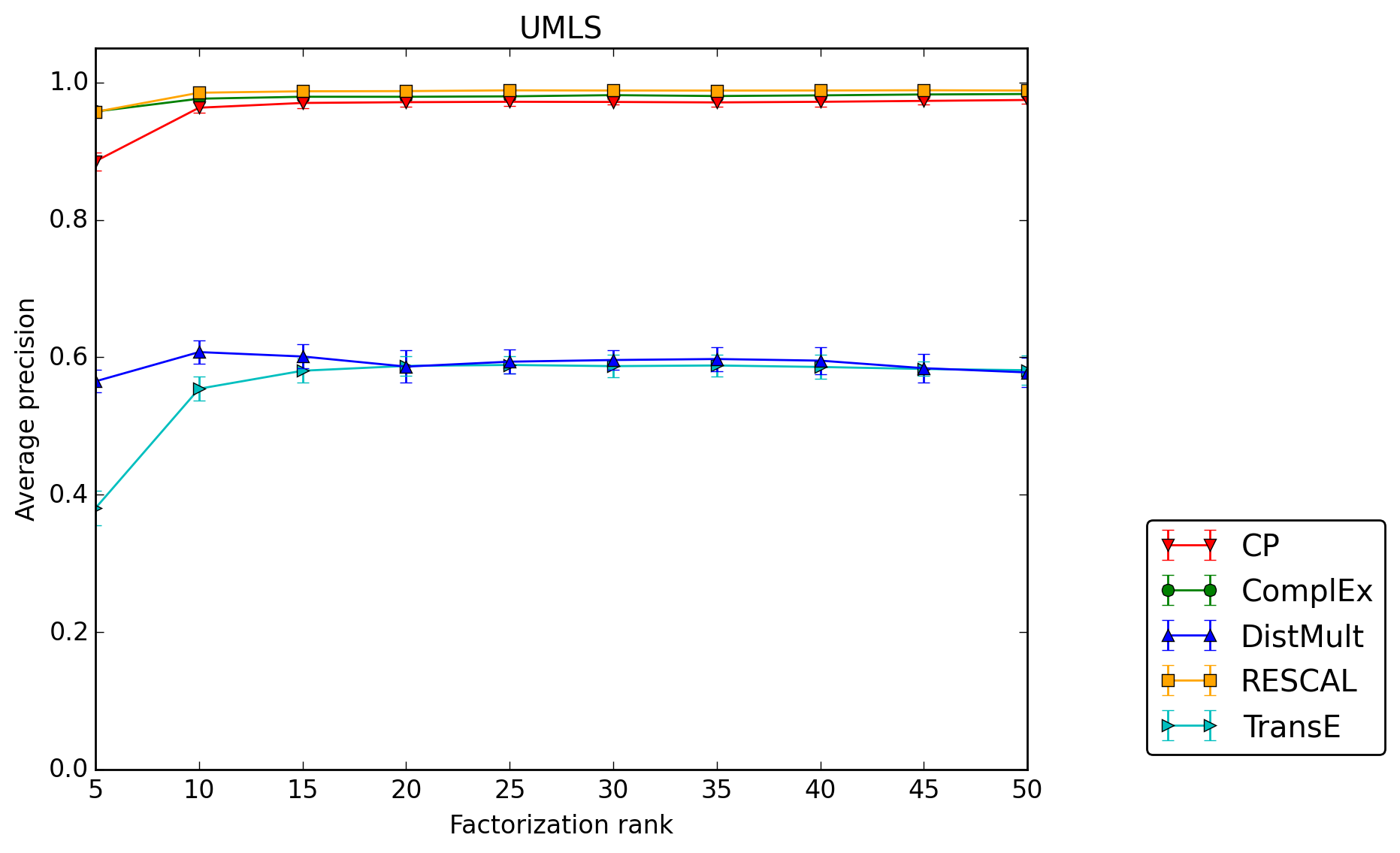} 
	\caption{Average precision (AP) for each factorization rank ranging from 5 to 50 for different state-of-the-art models on the Kinships data set (top) and on the UMLS data set (bottom).}
	\label{fig:exp_kinships_umls}
\end{figure}

On both data sets \textsc{ComplEx}, RESCAL and CP are very close, with a slight advantage for \textsc{ComplEx} on Kinships, and for RESCAL on UMLS. 
\textsc{DistMult} performs poorly here as many relations are antisymmetric
both in UMLS (causal links, anatomical hierarchies)
and Kinships (being father, uncle or grand-father). 

The fact that CP, RESCAL and \textsc{ComplEx} 
work so well on these data sets illustrates
the importance of having an expressive enough model,
as \textsc{DistMult} fails because of \mbox{antisymmetry}; 
the power of the multilinear product---that is the
tensor factorization approach---as \textsc{TransE} can be 
seen as a sum of bilinear products 
\cite{garcia2016combining}; but not yet the importance
of having unique entity embeddings, as CP works well. 
We believe having separate subject and object-entity embeddings
works well under the closed-world assumption, because of the
amount of training data compared to the number of embeddings
to learn.
Though when only a fractions of the positive training
examples are observed (as it is most often the case),
we will see in the next experiments 
that enforcing unique entity
embeddings is key to good generalization.

\subsection{Real Sparse Data Sets: FB15K and WN18}

Finally, we evaluated \textsc{ComplEx} on the FB15K and WN18 data sets, as they are well 
established benchmarks for the link prediction task. 
FB15K is a subset of Freebase \cite{Bollacker2008}, a curated knowledge graph of general facts, whereas WN18 is a subset of WordNet \cite{fellbaum1998wordnet}, a database featuring lexical relations between words.  We used the same training, validation and test set splits as in \citet{bordes2013translating}. Table \ref{tab:fb15k_wn18_meta} summarizes the metadata of the two data sets.

\begin{table}[t]
    \centering
    \begin{tabular}{l|rrrrr}
        & & & \multicolumn{3}{c}{Number of triples in sets:} \\
        Data set & $|\EntitySpace|$ & $|\RelationSpace|$ & Training & Validation & Test\\ \hline
        WN18 & 40,943 & 18 & 141,442 & 5,000 & 5,000 \\
        FB15K & 14,951 & 1,345 & 483,142 & 50,000 & 59,071 \\
    \end{tabular}
    \caption{Number of entities $|\EntitySpace|$, relations $|\RelationSpace|$, and observed triples in each split for the FB15K and WN18 data sets.}
    \label{tab:fb15k_wn18_meta}
\end{table}

\subsubsection{Experimental Setup}

As both data sets contain only positive triples, we generated negative samples
using the local closed-world assumption, as described in Section
\ref{sec:algo}. For evaluation, we measure the quality of the ranking of each test triple among all possible subject and object substitutions
: $r(s',o)$ and $r(s,o')$, for each $s', o'$ in $\setent$, as used in previous studies \cite{bordes2013translating,nickel_2016_holographic}.
 Mean Reciprocal Rank (MRR) and Hits at $N$ are standard evaluation measures for these data sets and come in two flavours: raw and filtered. The filtered metrics are computed \emph{after} removing all the other positive observed triples that appear in either training, validation or test set from the ranking, whereas the raw metrics do not remove these. 

Since ranking measures are used, previous studies generally preferred a max-margin ranking loss for the task \cite{bordes2013translating,nickel_2016_holographic}. We chose to use the negative log-likelihood of the logistic model---as described in the previous section---as it is a continuous surrogate of the sign-rank, and has been shown to learn compact representations for several important relations, especially for transitive relations~\cite{bouchard2015approximate}. As previously stated, we tried both losses in preliminary work, and indeed training the models with the log-likelihood yielded better results than with the max-margin ranking loss, especially on FB15K---except with \textsc{TransE}.

We report both filtered and raw MRR, and filtered Hits at 1, 3 and 10 in Table \ref{tab:fb15k_wn18_res} for the evaluated models.
The \textsc{HolE} model has recently been shown to be equivalent to \textsc{ComplEx} \cite{hayashi2017equivalence}, we record the original results for \textsc{HolE} as reported in \citet{nickel_2016_holographic} and briefly discuss the discrepancy of results
obtained with \textsc{ComplEx}. 

Reported results are given for the best set of hyper-parameters evaluated on the validation set for each model, after a distributed grid-search on the following values: $\rank \in \{$10, 20, 50, 100, 150, 200$\}$, $\lambda \in \{$0.1, 0.03, 0.01, 0.003, 0.001, 0.0003, 0.0$\}$, $\alpha \in \{$1.0, 0.5, 0.2, 0.1, 0.05, 0.02, 0.01$\}$,  $\eta \in \{$1, 2, 5, 10$\}$ with $\lambda$ the $L^2$ regularization parameter, $\alpha$ the initial learning rate, and $\eta$ the number of negatives generated per positive training triple. We also tried varying the batch size but this had no impact and we settled with 100 batches per epoch. 
With the best hyper-parameters, training the \textsc{ComplEx} model on a single GPU (NVIDIA Tesla P40) takes 45 minutes on WN18 ($K=150, \eta=1$), and three hours on FB15K ($K=200, \eta=10$).

\subsubsection{Results}
\label{sec:fb15k_res}

\begin{table*}[t]
    \centering
    \resizebox{\columnwidth}{!}{%
    \begin{tabular}{@{\extracolsep{8pt}}lllllllllll@{}}
        \toprule
        
         & \multicolumn{5}{c}{\textbf{WN18}} & \multicolumn{5}{c}{\textbf{FB15K}} \\ \cline{2-6} \cline{7-11}
         & \multicolumn{2}{c}{MRR} & \multicolumn{3}{c}{Hits at} & \multicolumn{2}{c}{MRR} & \multicolumn{3}{c}{Hits at} \\ \cline{2-3} \cline{4-6} \cline{7-8} \cline{9-11}
        
        Model & Filtered & Raw & 1 & 3 & 10 & Filtered & Raw & 1 & 3 & 10 \\ \hline
        
        CP & 0.075 & 0.058 & 0.049 & 0.080 & 0.125 & 0.326 & 0.152 & 0.219 & 0.376 & 0.532 \\
        \textsc{TransE} & 0.454 & 0.335 & 0.089 & 0.823 & 0.934 & 0.380 & 0.221 & 0.231 & 0.472 &  0.641 \\
        RESCAL & 0.894 & 0.583 & 0.867 & 0.918 & 0.935 & 0.461 & 0.226 & 0.324 & 0.536 & 0.720 \\
        \textsc{DistMult} & 0.822 & 0.532 & 0.728 & 0.914 & 0.936 & 0.654 & \textbf{0.242} & 0.546 & 0.733   & 0.824 \\
        \textsc{HolE}* & 0.938 & \textbf{0.616} & 0.930 & \textbf{0.945} & \textbf{0.949} & 0.524 & 0.232 & 0.402 & 0.613 & 0.739\\ \hline
        \textsc{ComplEx} & \textbf{0.941} &  0.587 &  \textbf{0.936} &  \textbf{0.945} &  0.947 & \textbf{0.692} & \textbf{0.242} & \textbf{0.599} & \textbf{0.759}   & \textbf{0.840} \\
        
        \bottomrule
    \end{tabular}
    }
    \caption{Filtered and raw mean reciprocal rank (MRR) for the models tested on the FB15K and WN18 data sets. Hits@N metrics are filtered. *Results reported from \citet{nickel_2016_holographic} for the \textsc{HolE} model, that has been shown to be equivalent to \textsc{ComplEx} \cite{hayashi2017equivalence}, score divergence on FB15K is only due to the loss function used \cite{trouillon2017comparison}.}
    \label{tab:fb15k_wn18_res}
\end{table*}

WN18 describes lexical and semantic hierarchies between concepts and contains many antisymmetric relations such as hypernymy, hyponymy, and being part of. Indeed, the \textsc{DistMult} and \textsc{TransE} models are outperformed here by \textsc{ComplEx} and \textsc{HolE}, which are on a par with respective filtered MRR scores of 0.941 and 0.938, which is expected as both
models are equivalent.

Table \ref{tab:wn18_detailed_res} shows the filtered MRR for the reimplemented models and each relation of WN18, confirming the advantage of \textsc{ComplEx} on antisymmetric relations while losing nothing on the others. 2D projections of the relation embeddings (Figures \ref{fig:pca12} \& \ref{fig:pca34}) visually corroborate the results.

\begin{table}
\centering
    \begin{tabular}{l|lllll}
        
        Relation name & \textsc{ComplEx} & RESCAL & \textsc{DistMult} & \textsc{TransE} & CP\\ \hline
        hypernym  & \textbf{0.953} & 0.935 & 0.791 & 0.446 & 0.109\\
        hyponym  & \textbf{0.946} & 0.932 & 0.710 & 0.361 & 0.009\\
        member\_meronym  & \textbf{0.921} & 0.851 & 0.704 & 0.418 & 0.019\\
        member\_holonym  & \textbf{0.946} & 0.861 & 0.740 & 0.465 & 0.134\\
        instance\_hypernym  & \textbf{0.965} & 0.833 & 0.943 & 0.961 & 0.233\\
        instance\_hyponym  & \textbf{0.945} & 0.849 & 0.940 & 0.745 & 0.040\\
        has\_part  & \textbf{0.933} & 0.879 & 0.753 & 0.426 & 0.035\\
        part\_of  & \textbf{0.940} & 0.888 & 0.867 & 0.455 & 0.094\\
        member\_of\_domain\_topic  & \textbf{0.924} & 0.865 & 0.914 & 0.861 & 0.007\\
        synset\_domain\_topic\_of  & \textbf{0.930} & 0.855 & 0.919 & 0.917 & 0.153\\
        member\_of\_domain\_usage  & \textbf{0.917} & 0.629 & \textbf{0.917} & 0.875 & 0.001\\
        synset\_domain\_usage\_of  & \textbf{1.000} & 0.541 & \textbf{1.000} & \textbf{1.000} & 0.134\\
        member\_of\_domain\_region  & \textbf{0.865} & 0.632 & 0.635 & \textbf{0.865} & 0.001\\
        synset\_domain\_region\_of  & 0.919 & 0.655 & 0.888 & \textbf{0.986} & 0.149\\
        derivationally\_related\_form  & \textbf{0.946} & 0.928 & 0.940 & 0.384 & 0.100\\
        similar\_to  & \textbf{1.000} & 0.001 & \textbf{1.000} & 0.244 & 0.000\\
        verb\_group  & \textbf{0.936} & 0.857 & 0.897 & 0.323 & 0.035\\
        also\_see  & 0.603 & 0.302 &\textbf{0.607} & 0.279 & 0.020\\

    \end{tabular}
    \caption{Filtered Mean Reciprocal Rank (MRR) for the models tested on each relation of the WordNet data set (WN18).}
    \label{tab:wn18_detailed_res}
    
    \vspace{-5mm}
\end{table}

On FB15K, the gap is much more pronounced and the \textsc{ComplEx} model largely outperforms \textsc{HolE}, with a filtered MRR of 0.692 and 59.9\% of Hits at 1, compared to 0.524 and 40.2\% for \textsc{HolE}.
This difference of scores between the two models, though they have been proved
to be equivalent \cite{hayashi2017equivalence}, is due to the use of the aforementioned
max-margin loss in the original \textsc{HolE} publication \cite{nickel_2016_holographic}
that performs worse than the log-likelihood on this dataset,
and to the generation of more than one negative
sample per positive in these experiments. This has been confirmed
and discussed in details by \citet{trouillon2017comparison}. 
The fact that \textsc{DistMult} yields fairly high scores (0.654 filtered MRR)
is also due to the task itself and the evaluation measures used. 
As the dataset only involves true facts, 
the test set never includes the opposite facts $r(o,s)$ of each
test fact $r(s,o)$ for \emph{antisymmetric} relations---as the opposite fact is always false. 
Thus highly scoring the opposite fact barely impacts the
rankings for antisymmetric relations. This is not the case in the fully observed 
experiments (Section \ref{sec:kinships_umls}), as the opposite fact is known to be false---for antisymmetric relations---and largely impacts the average precision of the \textsc{DistMult} model (Figure \ref{fig:exp_kinships_umls}).

RESCAL, that represents each relation with a $K \times K$ matrix, performs well on WN18
as there are few relations and hence not so many parameters. On FB15K though,
it probably overfits due to the large number of relations and thus
the large number of parameters to learn, and performs worse than a 
less expressive model like \textsc{DistMult}.
On both data sets, \textsc{TransE} and CP are largely left behind. This illustrates again the power of the multilinear product in the first case, and the importance of learning unique entity embeddings in the second. CP performs especially poorly on WN18 due to the small number of \mbox{relations}, which magnifies this subject/object difference. 

Figure \ref{fig:mrr_vs_rank} shows that the filtered MRR of the \textsc{ComplEx} model quickly 
converges on both data sets, showing that the low-rank hypothesis is reasonable in practice. 
The little gain of performances for ranks comprised between $50$ and $200$ also shows that 
\textsc{ComplEx} does not perform better because it has twice as many parameters for the 
same rank---the real and imaginary parts---compared to other linear space complexity models 
but indeed thanks to its better expressiveness.

\begin{figure}[t]
	\centering
	\includegraphics[width=0.70\linewidth]{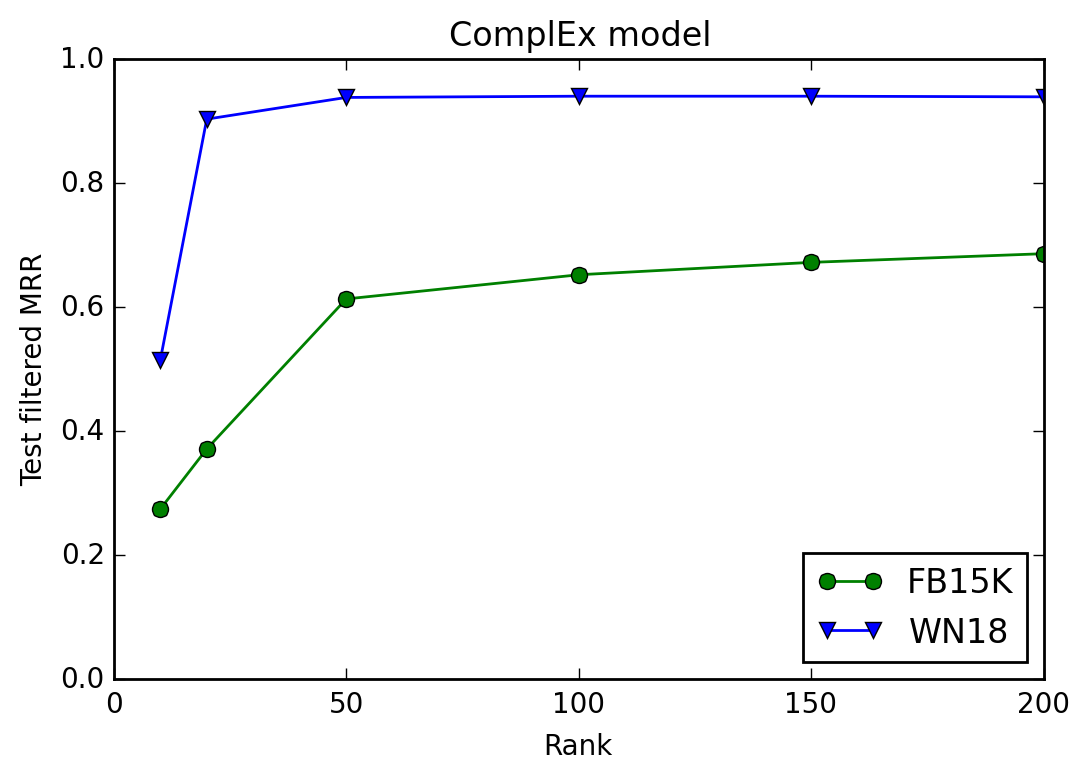}
	
    \caption{Best filtered MRR for \textsc{ComplEx} on the FB15K and WN18 data sets for different ranks. Increasing the rank gives little performance gain for ranks of $50-200$.}
	\label{fig:mrr_vs_rank}
	
\end{figure}

Best ranks were generally 150 or 200, in both cases scores were always very close for all models, suggesting there was no need to grid-search on higher ranks. The number of negative samples per positive sample also had a large influence on the filtered MRR on FB15K (up to +0.08 improvement from 1 to 10 negatives), but not much on WN18. On both data sets regularization was important (up to +0.05 on filtered MRR between $\lambda=0$ and optimal one).
We found the initial learning rate to be very important on FB15K, while not so much on WN18. We think this may also explain the large gap of improvement \textsc{ComplEx} provides on this data set compared to previously published results---as \textsc{DistMult} results are also better than those previously reported \cite{Yang2015}---along with the use of the log-likelihood objective. It seems that in general AdaGrad is relatively insensitive to the initial learning rate, perhaps causing some overconfidence in its ability to tune the step size online and consequently leading to less efforts when selecting the initial step size.

\subsection{Training time}
As defended in \Cref{sec:mat_case}, having a linear time and space complexity
becomes critical when the dataset grows. To illustrate this,
we report in \Cref{fig:training times} the evolution
of the filtered MRR on the validation set as a function of time,
for the best set of validated hyper-parameters for each model.
The convergence criterion used is the decrease of the validation filtered 
MRR---computed every 50 iterations---with a maximum number of iterations of 1000 (see \Cref{SGDC}).
All models have a linear complexity except for RESCAL that
has a quadratic one in the rank of the decomposition,
as it learns one matrix embedding for each relation $r \in \RelationSpace$.
Timings are measured on a single NVIDIA Tesla P40 GPU.

On WN18, all models reach convergence in a reasonable time, between 15 minutes
and 1 hour and 20 minutes. The difference between RESCAL and the other models
is not sharp there, first because its optimal embedding size ($K=50$) 
is lower compared to the other models. 
Secondly, there are only $|\RelationSpace| = 18$ relations
in WN18, hence the memory footprint of RESCAL is pretty similar
to the other models---because it represents only relations with
matrices and not entities.
On FB15K, the difference is much more pronounced, as RESCAL optimal rank
is similar to the other models; and with $|\RelationSpace| = 1345$ relations,
RESCAL has a much higher memory footprint, which implies more 
processor cache misses due to the uniformly-random nature of the SGD sampling.

RESCAL took more than
four days to train on FB15K, whereas other models took between 40 minutes and
3 hours. While a few days might seem manageable, this could not be
the case on larger data sets, as FB15K is but a small
subset of Freebase that contains $|\RelationSpace| = 35000$ relations
\cite{Bollacker2008}. This experimentally supports our claim
that linear complexity is required for scalability.

\begin{figure}[t]
	\centering
	\includegraphics[width=0.75\textwidth]{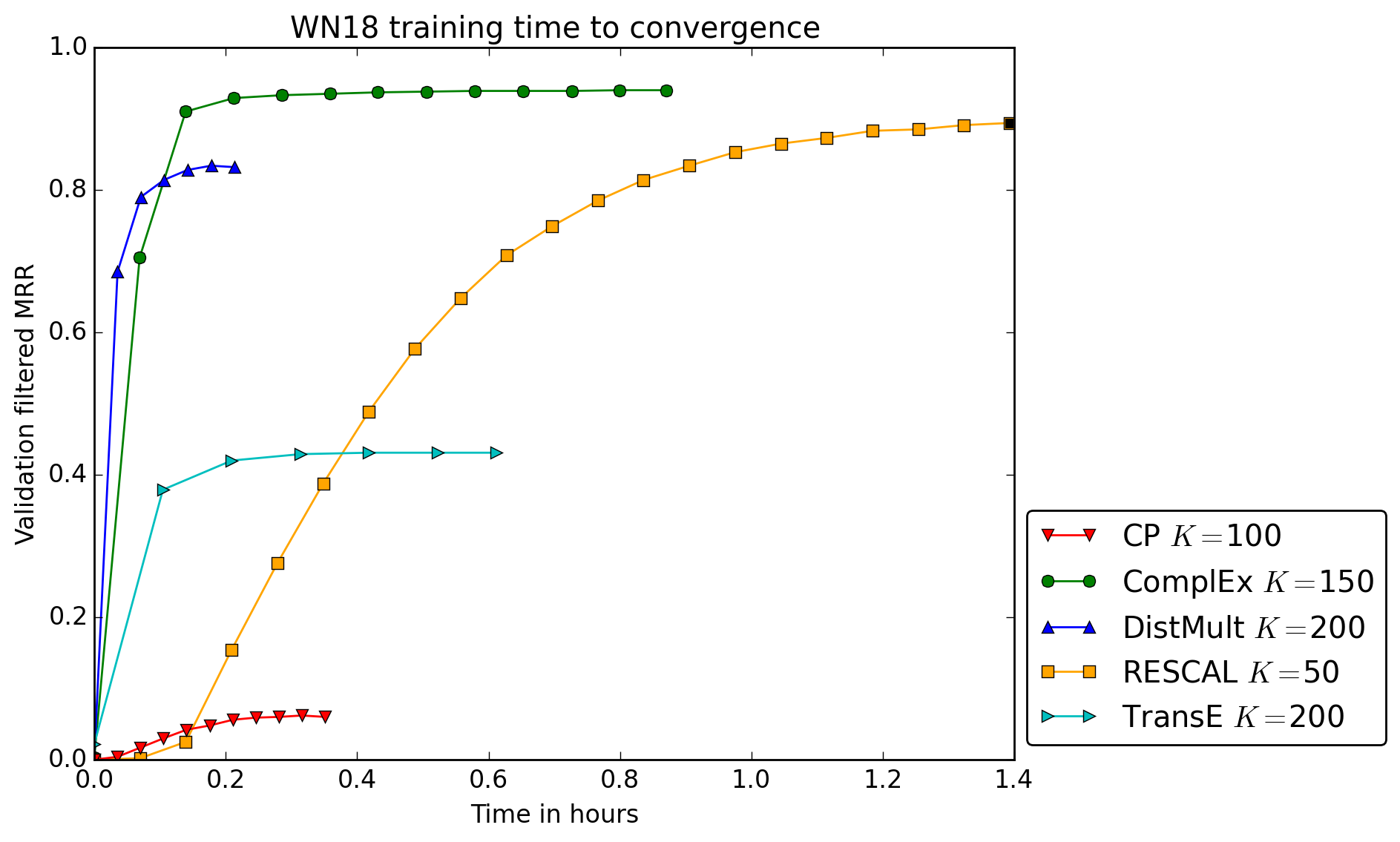} 
	\includegraphics[width=0.75\textwidth]{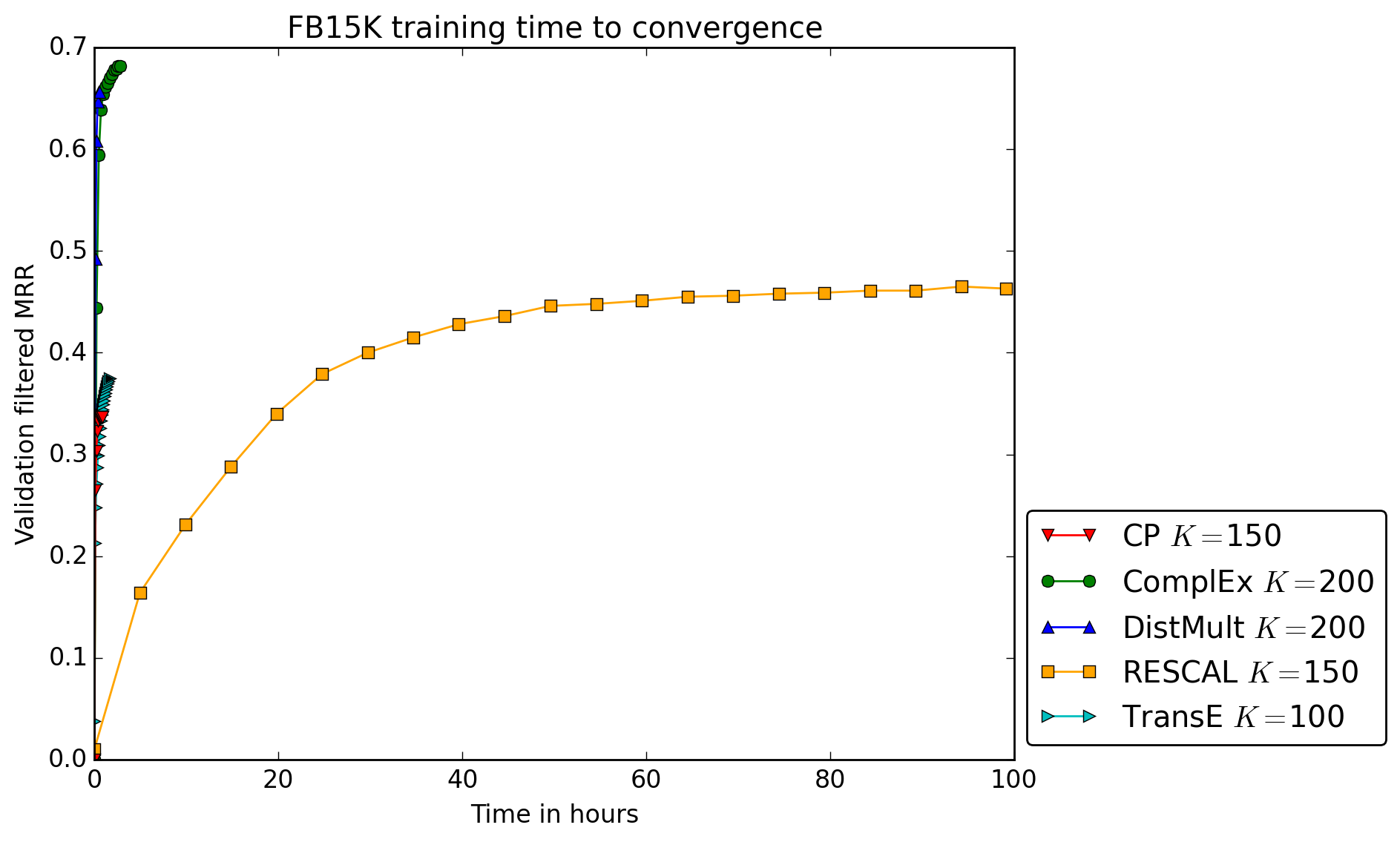}
	\caption{Evolution of the filtered MRR on the validation set as a function of time, on WN18 (top) and FB15K (bottom) for each model for its best set of hyper-parameters. The best rank $K$ is reported in legend. Final black marker indicates that the maximum number of iterations (1000) has been reached (RESCAL on WN18, \textsc{TransE} on FB15K).}
	\label{fig:training times}
\end{figure}

\subsubsection{Influence of Negative Samples}
We further investigated the influence of the number of negatives generated per positive training sample. In the previous experiment, due to computational limitations, the number of negatives per training sample, $\eta$, was validated over the set $\{1, 2, 5, 10\}$. On WN18 it proved to be of no help to have more than one generated negative per positive. Here we explore in which proportions increasing the number of generated negatives leads to better results on FB15K.
To do so, we fixed the best validated $\lambda, \rank, \alpha$ obtained from the previous experiment. We then let $\eta$ vary in $\{1, 2, 5, 10, 20, 50, 100, 200\}$.


\begin{figure}[t]
	\centering
	\includegraphics[width=0.70\linewidth]{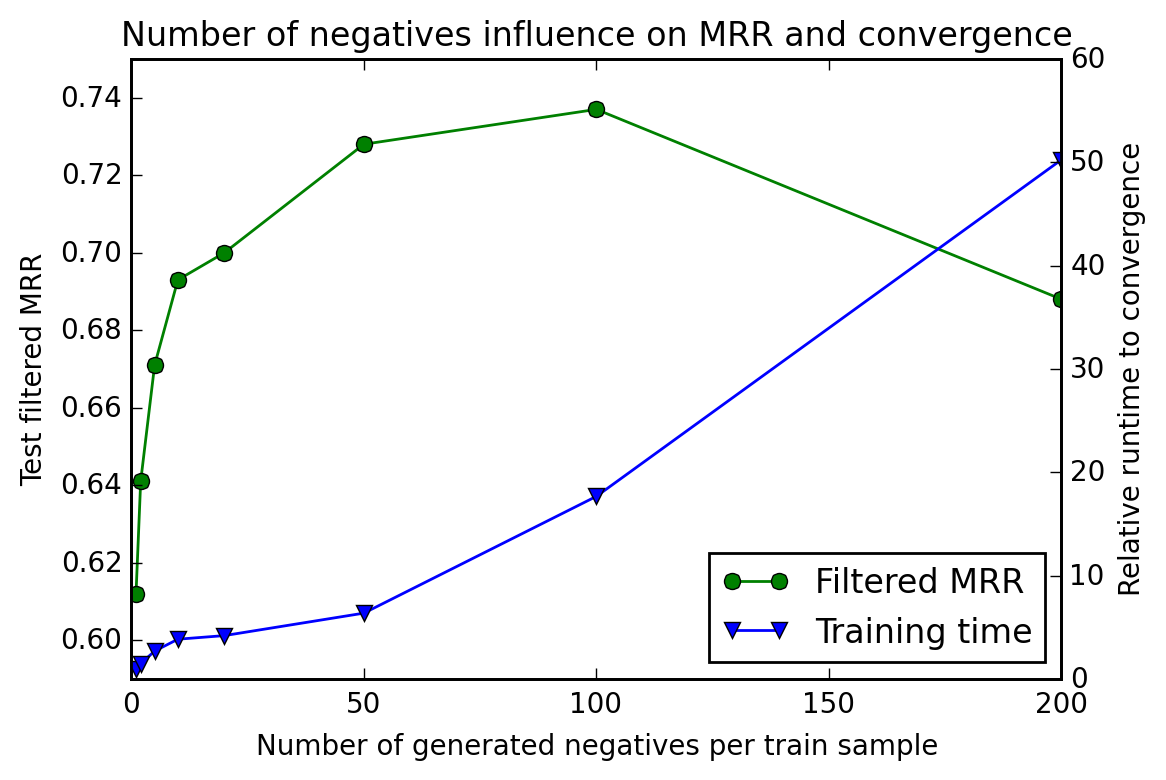}
	
    \caption{Influence of the number of negative triples generated per positive training example on the filtered test MRR and on training time to convergence on FB15K for the \textsc{ComplEx} model with $\rank=200$, $\lambda=0.01$ and $\alpha=0.5$. Times are given relative to the training time with one negative triple generated per positive training sample ($=1$ on time scale).}
	\label{fig:neg_ratio}
	
\end{figure}

Figure \ref{fig:neg_ratio} shows the influence of the number of generated negatives per positive training triple on the performance of \textsc{ComplEx} on FB15K.
Generating more negatives clearly improves the results up to 100 negative triples, with a filtered MRR of 0.737 and 64.8\% of Hits@1, before decreasing again with 200 negatives, probably due to the too large class imbalance. The model also converges with fewer epochs, which compensates partially for the additional training time per epoch, up to 50 negatives. It then grows linearly as the number of negatives increases.

\subsubsection{WN18 Embeddings Visualization}
\label{app:wn18_pca}

\begin{figure*}[!ht]
	\centering
	\includegraphics[width=0.8\textwidth]{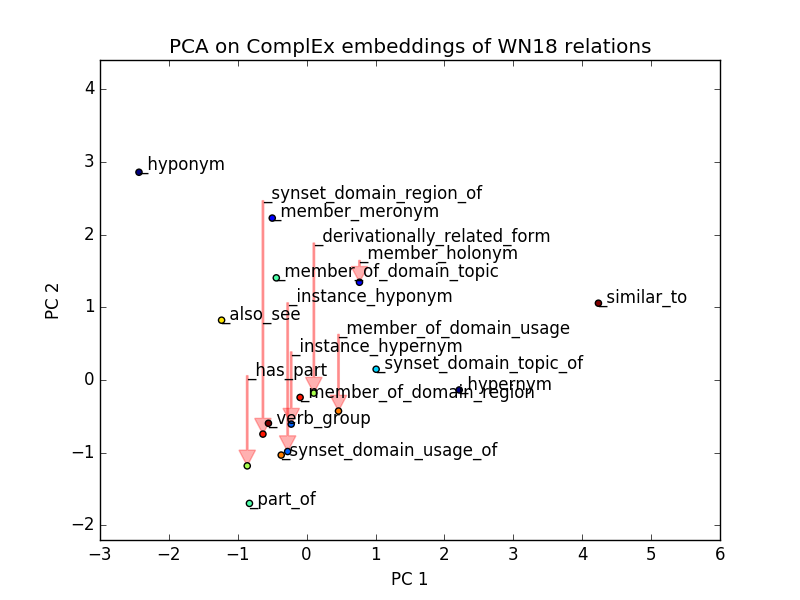}
	\includegraphics[width=0.8\textwidth]{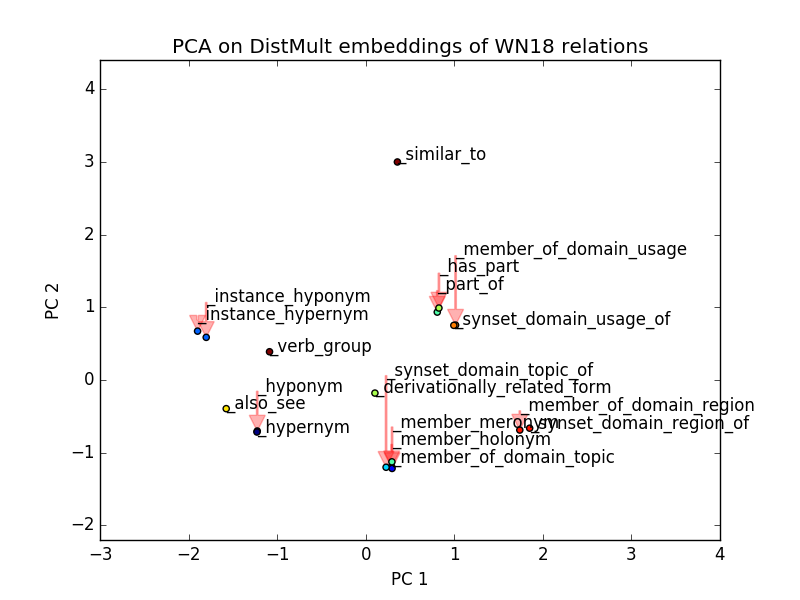}
	\vspace{-5mm}
	\caption{Plots of the first and second  components of the WN18 relations embeddings using principal component analysis. Red arrows link the labels to their point. Top: \textsc{ComplEx} embeddings. Bottom: \textsc{DistMult} embeddings. Opposite relations are clustered together by \textsc{DistMult} while correctly separated by \textsc{ComplEx}.}
	\label{fig:pca12}
\end{figure*}

\begin{figure*}[!ht]
	\centering
	\includegraphics[width=0.8\textwidth]{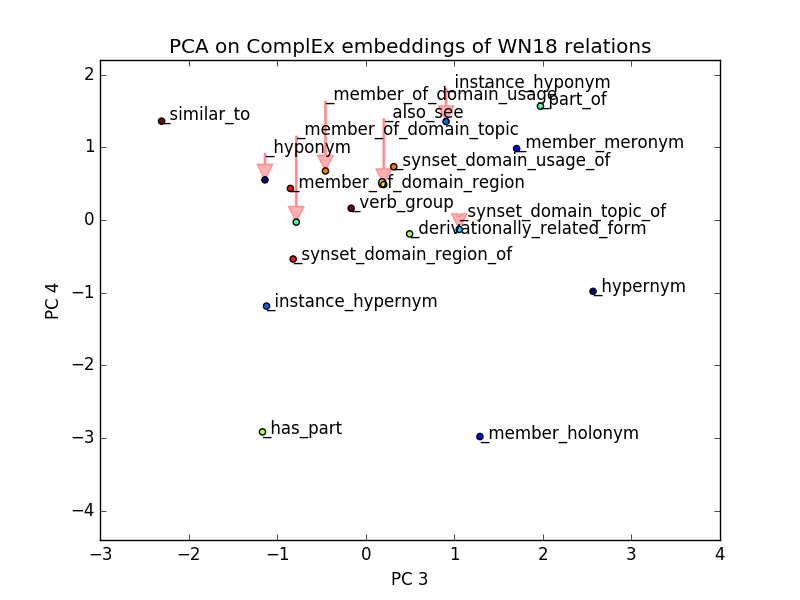}
	\includegraphics[width=0.8\textwidth]{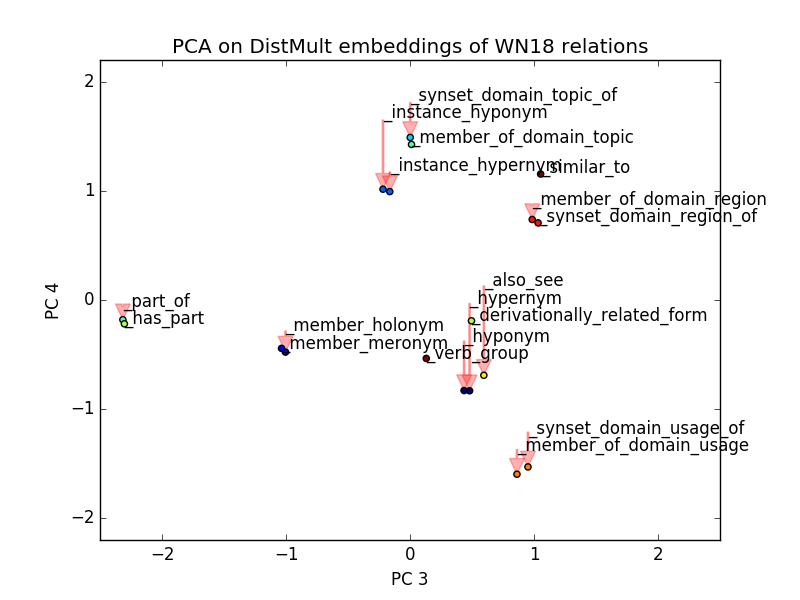}
	\vspace{-5mm}
	\caption{Plots of the third and fourth components of the WN18 relations embeddings using principal component analysis. Red arrows link the labels to their point. Top: \textsc{ComplEx} embeddings. Bottom: \textsc{DistMult} embeddings. Opposite relations are clustered together by \textsc{DistMult} while correctly separated by \textsc{ComplEx}.}
	\label{fig:pca34}
\end{figure*}

We used principal component analysis (PCA) to visualize embeddings of the relations 
of the WordNet data set (WN18). We plotted the four first components of the best \textsc{DistMult}
and \textsc{ComplEx} model's embeddings in Figures \ref{fig:pca12} \& \ref{fig:pca34}. For the \textsc{ComplEx} model, we simply concatenated
the real and imaginary parts of each embedding. 

Most of WN18 relations describe hierarchies, and are thus antisymmetric.
Each of these hierarchic relations has its inverse relation in the data set. For example: \texttt{hypernym} / \texttt{hyponym},
\texttt{part\_of} / \texttt{has\_part}, \texttt{synset\_domain\_topic\_of} / \texttt{member\_of\_domain\_topic}.
Since \textsc{DistMult} is unable to model antisymmetry, it will correctly represent the nature
of each pair of opposite relations, but not the direction of the relations.
Loosely speaking, in the \texttt{hypernym} / \texttt{hyponym} pair the nature is 
sharing semantics,
and the direction is that one entity generalizes the semantics of the other. 
This makes \textsc{DistMult} representing the opposite relations with very close embeddings.
It is especially striking for the third and
fourth principal component (Figure \ref{fig:pca34}). Conversely, \textsc{ComplEx} manages to oppose spatially
the opposite relations.

\section{Related Work}
\label{sec:rel_work}
We first discuss related work about complex-valued matrix and tensor decompositions,
and then review other approaches for knowledge graph completion.

\subsection{Complex Numbers}

When factorization methods are applied, the representation of the decomposition
is generally chosen in
accordance with the data, despite the fact that most real square 
matrices only have eigenvalues in the complex domain.
Indeed in the machine learning community, the data is usually real-valued,
and thus eigendecomposition is used for symmetric matrices, or other
decompositions such as (real-valued) singular value decomposition \cite{beltrami1873sulle}, 
non-negative matrix factorization \cite{paatero1994positive}, or canonical polyadic
decomposition when it comes to tensors \cite{hitchcock-sum-1927}.

Conversely, in signal processing, data is often complex-valued \cite{stoica2005spectral}
and the complex-valued counterparts of these decompositions are then used.
Joint diagonalization is also a much more common tool than in machine learning
for decomposing sets of (complex) dense square matrices
\cite{belouchrani1997blind,de2001independent}.

Some works on recommender systems use complex numbers as an encoding facility,
to merge two real-valued relations, 
similarity and liking, into one single complex-valued matrix which is
then decomposed with complex embeddings \cite{kunegis2012online,xie2015link}. 
Still, unlike our work, it is not real
data that is decomposed in the complex domain.

In deep learning, \citet{danihelka2016associative} proposed an LSTM 
extended with an associative memory based on complex-valued vectors for memorization tasks,
and \citet{hu2016initial} a complex-valued neural network
for speech synthesis. In both cases again, the data is first encoded in complex
vectors that are then fed into the network.

Conversely to these contributions, 
this work suggests that processing real-valued data
with complex-valued representation, through a projection
onto the real-valued subspace, can be a very simple way of 
increasing the expressiveness of the model considered.

\subsection{Knowledge Graph Completion}

Many knowledge graphs have recently arisen, pushed by the W3C recommendation
to use the resource description framework (RDF) \cite{cyganiak2014rdf} for 
data representation.
Examples of such knowledge graphs include DBPedia~\cite{dbpedia}, 
Freebase~\cite{Bollacker2008} and the Google Knowledge Vault
\cite{Dong:2014:KnowledgeVault}. Motivating applications of knowledge graph completion
include question answering \cite{bordes2014open} and more generally
probabilistic querying of knowledge bases \cite{huang2009query,krompass2014querying}.

First approaches to relational learning relied upon probabilistic graphical models \cite{Getoor2007}, such as bayesian networks \cite{Friedman1999} and markov 
logic networks \cite{richardson2006markov,raedt2016statistical}.

With the first embedding models, asymmetry of relations was quickly seen as 
a problem and asymmetric extensions of tensors were studied, mostly by either
considering independent embeddings~\cite{franz2009triplerank} or considering relations as
matrices instead of vectors in the RESCAL model~\cite{Nickel2011},
or both \cite{Sutskever2009}.
Direct extensions were based on uni-,bi- and trigram latent factors for triple data \cite{garcia2016combining}, as well as a low-rank relation matrix~\cite{Jenatton2012}.
\citet{bordes2014semantic} propose a two-layer model where subject and object embeddings
are first separately combined with the relation embedding, then each intermediate
representation is combined into the final score.

Pairwise interaction models were also considered to improve prediction performances. For example, the Universal Schema approach~\cite{riedel_2013_univschema} factorizes a 2D unfolding of the tensor (a matrix of entity pairs vs. relations) while \citet{Welbl2016} extend this also to other pairs.
\citet{riedel_2013_univschema} also consider augmenting the knowledge graph facts 
by exctracting them from textual data,
as does \citet{Toutanova2015}. Injecting prior knowledge in the form 
of Horn clauses in the objective loss of the Universal Schema model has also been considered \cite{Rocktaschel2015}. \citet{Chang2014} enhance the RESCAL model to take
into account information about the entity types. 
For recommender systems (thus with different subject/object sets of entities), \citet{baruch2014ternary} proposed a non-commutative
extension of the CP decomposition model.
More recently, Gaifman models that learn neighborhood embeddings of local
structures in the knowledge graph showed competitive performances
\cite{niepert2016discriminative}.

In the Neural Tensor Network (NTN) model, \citet{socher2013reasoning} combine linear transformations and multiple bilinear forms of subject and object embeddings to jointly feed them into a nonlinear neural layer. Its non-linearity and multiple ways of including interactions between embeddings gives it an advantage in expressiveness over models with simpler scoring function like \textsc{DistMult} or RESCAL. As a downside, its very large number of parameters can make the NTN model harder to train and overfit more easily.

The original multilinear \textsc{DistMult} model is symmetric in subject and object for every relation ~\cite{Yang2015} and achieves good performance on FB15K and WN18 data sets. However it is likely due
to the absence of true negatives in these data sets, as discussed in Section \ref{sec:fb15k_res}.

The \textsc{TransE} model from \citet{bordes2013translating} also embeds entities and relations in the same space and imposes a geometrical structural bias into the model: the subject entity vector should be close to the object entity vector once translated by the relation vector.

A recent novel way to handle antisymmetry is via the Holographic Embeddings (\textsc{HolE}) model by \citet{nickel_2016_holographic}. In \textsc{HolE} the circular correlation is used for combining entity embeddings, measuring the covariance between embeddings at different dimension shifts. This model has been shown to be equivalent to the \textsc{ComplEx} model \cite{hayashi2017equivalence,trouillon2017comparison}.

\section{Discussion and Future Work}

Though the decomposition proposed in this paper is clearly not unique, it is able to 
learn meaningful representations.
Still, characterizing all possible unitary diagonalizations that preserve
the real part is an interesting open question. Especially in 
an approximation setting with a constrained rank, in order to characterize the
decompositions that minimize a given reconstruction error.
That
might allow the creation of an iterative 
algorithm similar to eigendecomposition iterative methods
\cite{saad1992numerical} for computing such a
decomposition for any given real square matrix.

The proposed decomposition could also find applications in many other asymmetric square matrices
decompositions applications, such as spectral graph theory for 
directed graphs \cite{cvetkovic1997eigenspaces}, but also factorization of asymmetric
measures matrices such as asymmetric distance matrices \cite{mao2004modeling} 
and asymmetric similarity matrices \cite{pirasteh2015exploiting}.

From an optimization point of view, the objective function (Equation~(\ref{eq:objective})) 
is clearly non-convex,
and we could indeed not be reaching a globally optimal decomposition using 
stochastic gradient descent. Recent results show that 
there are no spurious local minima in the completion problem
of positive semi-definite matrix
\cite{ge2016matrix,bhojanapalli2016global}. Studying the extensibility
of these results to our decomposition is another possible line of future work.
The first step would be generalizing these results to 
symmetric real-valued matrix completion, then 
generalization to normal matrices should be straightforward.
The two last steps would be extending to matrices that are expressed as real
part of normal matrices, and finally to the joint decomposition 
of such matrices as a tensor.
We indeed noticed a remarkable stability of the scores across different 
random initialization of \textsc{ComplEx} for the same hyper-parameters, which suggests the
possibility of such theoretical property.

Practically, an obvious extension is to merge our approach with known extensions to tensor factorization models in order to further improve predictive performance. For example, the use of pairwise embeddings \cite{riedel_2013_univschema,Welbl2016} together with complex numbers might lead to improved results in many situations that involve non-compositionality. Adding bigram embeddings to the objective
could also improve the results as shown on other models \cite{garcia2016combining}.
Another direction would be to develop a more intelligent negative sampling procedure, to generate more informative negatives with respect to the positive sample from which they have been sampled. This would reduce the number of negatives required to reach good performance, thus accelerating training time.
Extension to relations between more than two entities, $n$-tuples,
is not straightforward, as \textsc{ComplEx}'s expressiveness comes from the complex conjugation
of the object-entity, that breaks the symmetry between the subject and object embeddings
in the scoring function. This stems from the Hermitian product, which seems to have no standard multilinear extension in the linear algebra literature, this question
hence remains largely open.

\section{Conclusion}

We described a new matrix and tensor decomposition with complex-valued latent factors called \textsc{ComplEx}.
The decomposition exists for all real square matrices, expressed
as the real part of normal matrices.
The result extends to sets of real square matrices---tensors---and answers to the 
requirements of the knowledge graph
completion task : handling a large variety of different relations including antisymmetric and
asymmetric ones, while being scalable.
Experiments confirm its theoretical versatility, as it substantially improves
over the state-of-the-art on real knowledge graphs.
It shows that real world relations can be efficiently approximated as
the real part of low-rank normal matrices.
The generality of the theoretical results and the effectiveness of the
experimental ones motivate for the application to other real square matrices
factorization problems. 
More generally, we hope that this paper will stimulate
the use of complex linear algebra in the machine learning community,
even and especially for processing real-valued data.


\acks{This work was supported in part by the Association Nationale de la
Recherche et de la Technologie through the CIFRE grant 2014/0121, in part by
the Paul Allen Foundation through an Allen Distinguished
Investigator grant, and in part by a Google Focused Research Award.
We would like to thank Ariadna Quattoni, St\'ephane Clinchant, Jean-Marc Andr\'eoli, Sofia Michel, Alejandro Blumentals, L\'eo Hubert and Pierre Comon for their helpful comments and feedback. }






\vskip 0.2in
\bibliography{nonauto_bib,complex_bib}

\begin{thebibliography}{65}
\providecommand{\natexlab}[1]{#1}
\providecommand{\url}[1]{\texttt{#1}}
\expandafter\ifx\csname urlstyle\endcsname\relax
  \providecommand{\doi}[1]{doi: #1}\else
  \providecommand{\doi}{doi: \begingroup \urlstyle{rm}\Url}\fi

\bibitem[Alon et~al.(2016)Alon, Moran, and Yehudayoff]{alon2016sign}
Noga Alon, Shay Moran, and Amir Yehudayoff.
\newblock Sign rank versus vc dimension.
\newblock In \emph{Conference on Learning Theory}, pages 47--80, 2016.

\bibitem[Auer et~al.(2007)Auer, Bizer, Kobilarov, Lehmann, and Ives]{dbpedia}
Sören Auer, Christian Bizer, Georgi Kobilarov, Jens Lehmann, and Zachary Ives.
\newblock {DBpedia}: A nucleus for a web of open data.
\newblock In \emph{International Semantic Web Conference, Busan, Korea}, pages
  11--15. Springer, 2007.

\bibitem[Baruch(2014)]{baruch2014ternary}
Guy Baruch.
\newblock A ternary non-commutative latent factor model for scalable three-way
  real tensor completion.
\newblock \emph{arXiv preprint arXiv:1410.7383}, 2014.

\bibitem[Belouchrani et~al.(1997)Belouchrani, Abed-Meraim, Cardoso, and
  Moulines]{belouchrani1997blind}
Adel Belouchrani, Karim Abed-Meraim, J-F Cardoso, and Eric Moulines.
\newblock A blind source separation technique using second-order statistics.
\newblock \emph{IEEE Transactions on Signal Processing}, 45\penalty0
  (2):\penalty0 434--444, 1997.

\bibitem[Beltrami(1873)]{beltrami1873sulle}
Eugenio Beltrami.
\newblock Sulle funzioni bilineari.
\newblock \emph{Giornale di Matematiche ad Uso degli Studenti Delle
  Universita}, 11\penalty0 (2):\penalty0 98--106, 1873.

\bibitem[Bergstra et~al.(2010)Bergstra, Breuleux, Bastien, Lamblin, Pascanu,
  Desjardins, Turian, Warde-Farley, and Bengio]{theano}
James Bergstra, Olivier Breuleux, Fr{\'{e}}d{\'{e}}ric Bastien, Pascal Lamblin,
  Razvan Pascanu, Guillaume Desjardins, Joseph Turian, David Warde-Farley, and
  Yoshua Bengio.
\newblock Theano: a {CPU} and {GPU} math expression compiler.
\newblock In \emph{Python for Scientific Computing Conference ({SciPy})}, June
  2010.

\bibitem[Bhojanapalli et~al.(2016)Bhojanapalli, Neyshabur, and
  Srebro]{bhojanapalli2016global}
Srinadh Bhojanapalli, Behnam Neyshabur, and Nathan Srebro.
\newblock Global optimality of local search for low rank matrix recovery.
\newblock \emph{arXiv preprint arXiv:1605.07221}, 2016.

\bibitem[Bollacker et~al.(2008)Bollacker, Evans, Paritosh, Sturge, and
  Taylor]{Bollacker2008}
Kurt Bollacker, Colin Evans, Praveen Paritosh, Tim Sturge, and Jamie Taylor.
\newblock {Freebase: a collaboratively created graph database for structuring
  human knowledge}.
\newblock In \emph{ACM SIGMOD International Conference on Management of Data},
  pages 1247--1250, 2008.

\bibitem[Bordes et~al.(2011)Bordes, Weston, Collobert, and
  Bengio]{bordes2011learning}
Antoine Bordes, Jason Weston, Ronan Collobert, and Yoshua Bengio.
\newblock Learning structured embeddings of knowledge bases.
\newblock In \emph{{AAAI} Conference on Artificial Intelligence}, 2011.

\bibitem[Bordes et~al.(2013{\natexlab{a}})Bordes, Usunier,
  Garc{\'{\i}}a{-}Dur{\'{a}}n, Weston, and Yakhnenko]{Bordes2013}
Antoine Bordes, Nicolas Usunier, Alberto Garc{\'{\i}}a{-}Dur{\'{a}}n, Jason
  Weston, and Oksana Yakhnenko.
\newblock Irreflexive and hierarchical relations as translations.
\newblock \emph{Computing Research Repository}, abs/1304.7158,
  2013{\natexlab{a}}.

\bibitem[Bordes et~al.(2013{\natexlab{b}})Bordes, Usunier, Garcia-Duran,
  Weston, and Yakhnenko]{bordes2013translating}
Antoine Bordes, Nicolas Usunier, Alberto Garcia-Duran, Jason Weston, and Oksana
  Yakhnenko.
\newblock Translating embeddings for modeling multi-relational data.
\newblock In \emph{Advances in Neural Information Processing Systems}, pages
  2787--2795, 2013{\natexlab{b}}.

\bibitem[Bordes et~al.(2014{\natexlab{a}})Bordes, Glorot, Weston, and
  Bengio]{bordes2014semantic}
Antoine Bordes, Xavier Glorot, Jason Weston, and Yoshua Bengio.
\newblock A semantic matching energy function for learning with
  multi-relational data.
\newblock \emph{Machine Learning}, 94\penalty0 (2):\penalty0 233--259,
  2014{\natexlab{a}}.

\bibitem[Bordes et~al.(2014{\natexlab{b}})Bordes, Weston, and
  Usunier]{bordes2014open}
Antoine Bordes, Jason Weston, and Nicolas Usunier.
\newblock Open question answering with weakly supervised embedding models.
\newblock In \emph{Joint European Conference on Machine Learning and Knowledge
  Discovery in Databases}, pages 165--180. Springer, 2014{\natexlab{b}}.

\bibitem[Bouchard et~al.(2015)Bouchard, Singh, and
  Trouillon]{bouchard2015approximate}
Guillaume Bouchard, Sameer Singh, and Th\'eo Trouillon.
\newblock On approximate reasoning capabilities of low-rank vector spaces.
\newblock \emph{AAAI Spring Symposium on Knowledge Representation and
  Reasoning: Integrating Symbolic and Neural Approaches}, 2015.

\bibitem[Cauchy(1829)]{cauchy1829}
Augustin-Louis Cauchy.
\newblock Sur l'\'equation \`a l'aide de laquelle on d\'etermine les
  in\'egalit\'es s\'eculaires des mouvements des plan\`etes.
\newblock \emph{\OE{}uvres compl\`etes, $II^e$ s\'erie}, 9:\penalty0 174--195,
  1829.

\bibitem[Chang et~al.(2014)Chang, Yih, Yang, and Meek]{Chang2014}
K.~W. Chang, W.~T. Yih, B.~Yang, and C.~Meek.
\newblock {Typed tensor decomposition of knowledge bases for relation
  extraction}.
\newblock In \emph{Conference on Empirical Methods on Natural Language
  Processing}, 2014.

\bibitem[Cvetkovi{\'c} et~al.(1997)Cvetkovi{\'c}, Rowlinson, and
  Simic]{cvetkovic1997eigenspaces}
Drago{\v{s}}~M. Cvetkovi{\'c}, Peter Rowlinson, and Slobodan Simic.
\newblock \emph{Eigenspaces of graphs}.
\newblock Number~66. Cambridge University Press, 1997.

\bibitem[Cyganiak et~al.(2014)Cyganiak, Wood, and Lanthaler]{cyganiak2014rdf}
Richard Cyganiak, David Wood, and Markus Lanthaler.
\newblock Rdf 1.1 concepts and abstract syntax.
\newblock \emph{W3C Recommendation}, 2014.

\bibitem[Danihelka et~al.(2016)Danihelka, Wayne, Uria, Kalchbrenner, and
  Graves]{danihelka2016associative}
Ivo Danihelka, Greg Wayne, Benigno Uria, Nal Kalchbrenner, and Alex Graves.
\newblock Associative long short-term memory.
\newblock \emph{arXiv preprint arXiv:1602.03032}, 2016.

\bibitem[De~Lathauwer et~al.(2001)De~Lathauwer, De~Moor, and
  Vandewalle]{de2001independent}
Lieven De~Lathauwer, Bart De~Moor, and Joos Vandewalle.
\newblock Independent component analysis and (simultaneous) third-order tensor
  diagonalization.
\newblock \emph{IEEE Transactions on Signal Processing}, 49\penalty0
  (10):\penalty0 2262--2271, 2001.

\bibitem[Denham(1973)]{denham1973detection}
Woodrow~W Denham.
\newblock \emph{The detection of patterns in Alyawara nonverbal behavior}.
\newblock PhD thesis, University of Washington, Seattle., 1973.

\bibitem[Dong et~al.(2014)Dong, Gabrilovich, Heitz, Horn, Lao, Murphy,
  Strohmann, Sun, and Zhang]{Dong:2014:KnowledgeVault}
Xin Dong, Evgeniy Gabrilovich, Geremy Heitz, Wilko Horn, Ni~Lao, Kevin Murphy,
  Thomas Strohmann, Shaohua Sun, and Wei Zhang.
\newblock Knowledge vault: A web-scale approach to probabilistic knowledge
  fusion.
\newblock In \emph{ACM SIGKDD International Conference on Knowledge Discovery
  and Data Mining}, KDD '14, pages 601--610, 2014.

\bibitem[Duchi et~al.(2011)Duchi, Hazan, and Singer]{duchi2011adaptive}
John Duchi, Elad Hazan, and Yoram Singer.
\newblock Adaptive subgradient methods for online learning and stochastic
  optimization.
\newblock \emph{Journal of Machine Learning Research}, 12:\penalty0 2121--2159,
  2011.

\bibitem[Fellbaum(1998)]{fellbaum1998wordnet}
Christiane Fellbaum.
\newblock \emph{WordNet}.
\newblock Wiley Online Library, 1998.

\bibitem[Franz et~al.(2009)Franz, Schultz, Sizov, and
  Staab]{franz2009triplerank}
Thomas Franz, Antje Schultz, Sergej Sizov, and Steffen Staab.
\newblock Triplerank: Ranking semantic web data by tensor decomposition.
\newblock In \emph{International Semantic Web Conference}, pages 213--228,
  2009.

\bibitem[Friedman et~al.(1999)Friedman, Getoor, Koller, and
  Pfeffer]{Friedman1999}
Nir Friedman, Lise Getoor, Daphne Koller, and Avi Pfeffer.
\newblock {Learning Probabilistic Relational Models}.
\newblock In \emph{International Joint Conference on Artificial Intelligence},
  number August, pages 1300--1309, 1999.
\newblock ISBN 3540422897.
\newblock \doi{10.1.1.101.3165}.

\bibitem[Garcia-Duran et~al.(2016)Garcia-Duran, Bordes, Usunier, and
  Grandvalet]{garcia2016combining}
Alberto Garcia-Duran, Antoine Bordes, Nicolas Usunier, and Yves Grandvalet.
\newblock Combining two and three-way embedding models for link prediction in
  knowledge bases.
\newblock \emph{Journal of Artificial Intelligence Research}, 55:\penalty0
  715--742, 2016.

\bibitem[Ge et~al.(2016)Ge, Lee, and Ma]{ge2016matrix}
Rong Ge, Jason~D Lee, and Tengyu Ma.
\newblock Matrix completion has no spurious local minimum.
\newblock \emph{arXiv preprint arXiv:1605.07272}, 2016.

\bibitem[Getoor and Taskar(2007)]{Getoor2007}
Lise Getoor and Ben Taskar.
\newblock \emph{Introduction to Statistical Relational Learning}.
\newblock The MIT Press, 2007.
\newblock ISBN 0262072882.

\bibitem[Hayashi and Shimbo(2017)]{hayashi2017equivalence}
Katsuhiko Hayashi and Masashi Shimbo.
\newblock On the equivalence of holographic and complex embeddings for link
  prediction.
\newblock \emph{arXiv preprint arXiv:1702.05563}, 2017.

\bibitem[Hitchcock(1927)]{hitchcock-sum-1927}
F.~L. Hitchcock.
\newblock The expression of a tensor or a polyadic as a sum of products.
\newblock \emph{J. Math. Phys}, 6\penalty0 (1):\penalty0 164--189, 1927.

\bibitem[Horn and Johnson(2012)]{horn2012matrix}
Roger~A Horn and Charles~R Johnson.
\newblock \emph{Matrix analysis}.
\newblock Cambridge University Press, 2012.

\bibitem[Hu et~al.(2016)Hu, Yamagishi, Richmond, Subramanian, and
  Stylianou]{hu2016initial}
Qiong Hu, Junichi Yamagishi, Korin Richmond, Kartick Subramanian, and Yannis
  Stylianou.
\newblock Initial investigation of speech synthesis based on complex-valued
  neural networks.
\newblock In \emph{IEEE International Conference on Acoustics, Speech and
  Signal Processing}, pages 5630--5634, 2016.

\bibitem[Huang and Liu(2009)]{huang2009query}
Hai Huang and Chengfei Liu.
\newblock Query evaluation on probabilistic rdf databases.
\newblock In \emph{International Conference on Web Information Systems
  Engineering}, pages 307--320. Springer, 2009.

\bibitem[Jenatton et~al.(2012)Jenatton, Bordes, Le~Roux, and
  Obozinski]{Jenatton2012}
Rodolphe Jenatton, Antoine Bordes, Nicolas Le~Roux, and Guillaume Obozinski.
\newblock {A Latent Factor Model for Highly Multi-relational Data}.
\newblock In \emph{Advances in Neural Information Processing Systems 25}, pages
  3167--3175, 2012.

\bibitem[Koren et~al.(2009)Koren, Bell, and Volinsky]{koren_netflix}
Yehuda Koren, Robert Bell, and Chris Volinsky.
\newblock Matrix factorization techniques for recommender systems.
\newblock \emph{Computer}, 42\penalty0 (8):\penalty0 30--37, 2009.

\bibitem[Krompa{\ss} et~al.(2014)Krompa{\ss}, Nickel, and
  Tresp]{krompass2014querying}
Denis Krompa{\ss}, Maximilian Nickel, and Volker Tresp.
\newblock Querying factorized probabilistic triple databases.
\newblock In \emph{International Semantic Web Conference}, pages 114--129,
  2014.

\bibitem[Kruskal(1989)]{kruskal1989rank}
Joseph~B Kruskal.
\newblock Rank, decomposition, and uniqueness for 3-way and n-way arrays.
\newblock In \emph{Multiway data analysis}, pages 7--18. North-Holland
  Publishing Co., 1989.

\bibitem[Kunegis et~al.(2012)Kunegis, Gr{\"o}ner, and
  Gottron]{kunegis2012online}
J{\'e}r{\^o}me Kunegis, Gerd Gr{\"o}ner, and Thomas Gottron.
\newblock Online dating recommender systems: The split-complex number approach.
\newblock In \emph{ACM RecSys Workshop on Recommender Systems and the Social
  Web}, pages 37--44. ACM, 2012.

\bibitem[Linial et~al.(2007)Linial, Mendelson, Schechtman, and
  Shraibman]{Linial2007}
Nati Linial, Shahar Mendelson, Gideon Schechtman, and Adi Shraibman.
\newblock {Complexity measures of sign matrices}.
\newblock \emph{Combinatorica}, 27\penalty0 (4):\penalty0 439--463, 2007.

\bibitem[Mao and Saul(2004)]{mao2004modeling}
Yun Mao and Lawrence~K Saul.
\newblock Modeling distances in large-scale networks by matrix factorization.
\newblock In \emph{ACM SIGCOMM conference on Internet Measurement}, pages
  278--287, 2004.

\bibitem[McCray(2003)]{mccray2003upper}
Alexa~T McCray.
\newblock An upper-level ontology for the biomedical domain.
\newblock \emph{Comparative and Functional Genomics}, 4\penalty0 (1):\penalty0
  80--84, 2003.

\bibitem[Nickel et~al.(2011)Nickel, Tresp, and Kriegel]{Nickel2011}
Maximilian Nickel, Volker Tresp, and Hans-Peter Kriegel.
\newblock A three-way model for collective learning on multi-relational data.
\newblock In \emph{International Conference on Machine Learning}, pages
  809--816, 2011.

\bibitem[Nickel et~al.(2014)Nickel, Jiang, and Tresp]{nickel2014reducing}
Maximilian Nickel, Xueyan Jiang, and Volker Tresp.
\newblock Reducing the rank in relational factorization models by including
  observable patterns.
\newblock In \emph{Advances in Neural Information Processing Systems}, pages
  1179--1187, 2014.

\bibitem[Nickel et~al.(2016{\natexlab{a}})Nickel, Murphy, Tresp, and
  Gabrilovich]{nickel_2016_review}
Maximilian Nickel, Kevin Murphy, Volker Tresp, and Evgeniy Gabrilovich.
\newblock A review of relational machine learning for knowledge graphs.
\newblock \emph{Proceedings of the {IEEE}}, 104\penalty0 (1):\penalty0 11--33,
  2016{\natexlab{a}}.

\bibitem[Nickel et~al.(2016{\natexlab{b}})Nickel, Rosasco, and
  Poggio]{nickel_2016_holographic}
Maximilian Nickel, Lorenzo Rosasco, and Tomaso~A. Poggio.
\newblock Holographic embeddings of knowledge graphs.
\newblock In \emph{{AAAI} Conference on Artificial Intelligence}, pages
  1955--1961, 2016{\natexlab{b}}.

\bibitem[Niepert(2016)]{niepert2016discriminative}
Mathias Niepert.
\newblock Discriminative gaifman models.
\newblock In \emph{Advances in Neural Information Processing Systems}, pages
  3405--3413, 2016.

\bibitem[Paatero and Tapper(1994)]{paatero1994positive}
Pentti Paatero and Unto Tapper.
\newblock Positive matrix factorization: A non-negative factor model with
  optimal utilization of error estimates of data values.
\newblock \emph{Environmetrics}, 5\penalty0 (2):\penalty0 111--126, 1994.

\bibitem[Pirasteh et~al.(2015)Pirasteh, Hwang, and
  Jung]{pirasteh2015exploiting}
Parivash Pirasteh, Dosam Hwang, and Jason~J Jung.
\newblock Exploiting matrix factorization to asymmetric user similarities in
  recommendation systems.
\newblock \emph{Knowledge-Based Systems}, 83:\penalty0 51--57, 2015.

\bibitem[Raedt et~al.(2016)Raedt, Kersting, Natarajan, and
  Poole]{raedt2016statistical}
Luc~De Raedt, Kristian Kersting, Sriraam Natarajan, and David Poole.
\newblock Statistical relational artificial intelligence: Logic, probability,
  and computation.
\newblock \emph{Synthesis Lectures on Artificial Intelligence and Machine
  Learning}, 10\penalty0 (2):\penalty0 1--189, 2016.

\bibitem[Richardson and Domingos(2006)]{richardson2006markov}
Matthew Richardson and Pedro Domingos.
\newblock Markov logic networks.
\newblock \emph{Machine Learning}, 62\penalty0 (1-2):\penalty0 107--136, 2006.

\bibitem[Riedel et~al.(2013)Riedel, Yao, McCallum, and
  Marlin]{riedel_2013_univschema}
Sebastian Riedel, Limin Yao, Andrew McCallum, and Benjamin~M. Marlin.
\newblock Relation extraction with matrix factorization and universal schemas.
\newblock In \emph{Human Language Technologies: Conference of the North
  American Chapter of the Association of Computational Linguistics}, pages
  74--84, 2013.

\bibitem[Rocktaschel et~al.(2015)Rocktaschel, Singh, and
  Riedel]{Rocktaschel2015}
T~Rocktaschel, S~Singh, and S~Riedel.
\newblock {Injecting Logical Background Knowledge into Embeddings for Relation
  Extraction}.
\newblock In \emph{Conference of the North American Chapter of the Association
  for Computational Linguistics}, pages 1119--1129, 2015.

\bibitem[Saad(1992)]{saad1992numerical}
Youcef Saad.
\newblock \emph{Numerical methods for large eigenvalue problems}, volume 158.
\newblock SIAM, 1992.

\bibitem[Sahoo et~al.(2009)Sahoo, Halb, Hellmann, Idehen, Thibodeau~Jr, Auer,
  Sequeda, and Ezzat]{sahoo2009survey}
Satya~S Sahoo, Wolfgang Halb, Sebastian Hellmann, Kingsley Idehen, Ted
  Thibodeau~Jr, S{\"o}ren Auer, Juan Sequeda, and Ahmed Ezzat.
\newblock A survey of current approaches for mapping of relational databases to
  rdf.
\newblock \emph{W3C RDB2RDF Incubator Group Report}, pages 113--130, 2009.

\bibitem[Socher et~al.(2013)Socher, Chen, Manning, and Ng]{socher2013reasoning}
Richard Socher, Danqi Chen, Christopher~D Manning, and Andrew Ng.
\newblock Reasoning with neural tensor networks for knowledge base completion.
\newblock In \emph{Advances in Neural Information Processing Systems}, pages
  926--934, 2013.

\bibitem[Stoica and Moses(2005)]{stoica2005spectral}
Petre Stoica and Randolph~L Moses.
\newblock \emph{Spectral analysis of signals}, volume 452.
\newblock Pearson Prentice Hall Upper Saddle River, NJ, 2005.

\bibitem[Sutskever(2009)]{Sutskever2009}
Ilya Sutskever.
\newblock Modelling relational data using bayesian clustered tensor
  factorization.
\newblock In \emph{Advances in Neural Information Processing Systems}, pages
  1--8, 2009.

\bibitem[Toutanova et~al.(2015)Toutanova, Pantel, and Gamon]{Toutanova2015}
Kristina Toutanova, Patrick Pantel, and Michael Gamon.
\newblock {Representing Text for Joint Embedding of Text and Knowledge Bases}.
\newblock In \emph{Conference on Empirical Methods on Natural Language
  Processing}, 2015.

\bibitem[Trouillon and Nickel(2017)]{trouillon2017comparison}
Th\'eo Trouillon and Maximilian Nickel.
\newblock Complex and holographic embeddings of knowledge graphs: a comparison.
\newblock \emph{International Workshop on Statistical Relational AI}, 2017.

\bibitem[Trouillon et~al.(2016)Trouillon, Welbl, Riedel, Gaussier, and
  Bouchard]{trouillon2016}
Th\'eo Trouillon, Johannes Welbl, Sebastian Riedel, \'Eric Gaussier, and
  Guillaume Bouchard.
\newblock {Complex embeddings for simple link prediction}.
\newblock In \emph{International Conference on Machine Learning}, volume~48,
  pages 2071--2080, 2016.

\bibitem[von Neumann(1929)]{vonneumann1929}
John von Neumann.
\newblock Zur algebra der funktionaloperationen und der theorie der normalen
  operatoren.
\newblock \emph{Mathematische Annalen}, 102:\penalty0 370--427, 1929.

\bibitem[Welbl et~al.(2016)Welbl, Bouchard, and Riedel]{Welbl2016}
Johannes Welbl, Guillaume Bouchard, and Sebastian Riedel.
\newblock A factorization machine framework for testing bigram embeddings in
  knowledge base completion.
\newblock In \emph{Workshop on Automated Knowledge Base Construction
  AKBC@NAACL-HLT}, pages 103--107, 2016.

\bibitem[Xie et~al.(2015)Xie, Chen, Shang, Feng, and Li]{xie2015link}
Feng Xie, Zhen Chen, Jiaxing Shang, Xiaoping Feng, and Jun Li.
\newblock A link prediction approach for item recommendation with complex
  numbers.
\newblock \emph{Knowledge-Based Systems}, 81:\penalty0 148--158, 2015.

\bibitem[Yang et~al.(2015)Yang, Yih, He, Gao, and Deng]{Yang2015}
Bishan Yang, Wen-Tau Yih, Xiaodong He, Jianfeng Gao, and Li~Deng.
\newblock {Embedding entities and relations for learning and inference in
  knowledge bases}.
\newblock In \emph{International Conference on Learning Representations}, 2015.

\end{thebibliography}

\end{document}